\documentclass[journal]{IEEEtran}
\usepackage{xspace,amsmath,amssymb,amsthm,amsfonts,epsfig,syntonly,dsfont}
\usepackage{cite,bm,color,url,textcomp}
\usepackage{algorithmicx,algorithm,algpseudocode}
\usepackage{epstopdf}
\usepackage{empheq,float}
\usepackage{graphicx,subfigure,balance}
\usepackage{multirow}

\newtheorem{lemma}{Lemma}
\newtheorem{corollary}{Corollary}
\newtheorem{theorem}{Theorem}

\usepackage[table]{xcolor}
\theoremstyle{definition}\newtheorem{remark}{Remark}

\def\baselinestretch{0.99}
\usepackage{times}

\begin{document}
\title{Secure Mobile Edge Computing in IoT \center{via Collaborative Online Learning}}
\author{Bingcong Li, Tianyi Chen, and Georgios B. Giannakis\\
	
%	$^*$Dept. of Elec. \& Comput. Engr. and Digital Technology Center, University of Minnesota, USA\\
\thanks {Work in this paper was supported by NSF 1509040, 1508993, and 1711471.
%\thanks {Copyright (c) 2013 IEEE. Personal use of this material is permitted. However, permission to use this material for any other purposes must be obtained from the IEEE by sending a request to pubs-permissions@ieee.org.}
B. Li, T. Chen and G. B. Giannakis are with the Department of Electrical and Computer Engineering and the Digital Technology Center, University of Minnesota, Minneapolis, MN 55455 USA. Emails: \{lixx5599, chen3827, georgios\}@umn.edu
}
}
\markboth{}{}
\maketitle

% The author names and affiliations should appear only in the accepted paper.
%
%\author{ {\bf Harry Q.~Bovik\thanks{Footnote for author to give an
%alternate address.}} \\
%Computer Science Dept. \\
%Cranberry University\\
%Pittsburgh, PA 15213 \\
%\And
%{\bf Coauthor}  \\
%Affiliation          \\
%Address \\
%\And
%{\bf Coauthor} 
%Affiliation \\
%Address    \\
%(if needed)\\
%}

\begin{abstract}
	To accommodate heterogeneous tasks in Internet of Things (IoT), a new communication and computing paradigm termed mobile edge computing emerges that extends computing services from the cloud to edge, but at the same time exposes new challenges on security. 
	The present paper studies online security-aware edge computing under jamming attacks. 
%	To guarantee reliability, existing anti-jamming solutions typically rely on utilizing extra resources such as spectrum and transmission power to evade jamming attacks, which is not desirable for IoT applications with scarcity in power and spectrum resources. 
%	We will take an alternative way to deal with jamming attacks. 
Leveraging online learning tools, novel algorithms abbreviated as SAVE-S and SAVE-A are developed to cope with the stochastic and adversarial forms of jamming, respectively. 
Without utilizing extra resources such as spectrum and transmission power to evade jamming attacks, 
	SAVE-S and SAVE-A can 	select the most reliable server to offload computing tasks with minimal privacy and security concerns.
It is analytically established that without any prior information on future jamming and server security risks, the proposed schemes can achieve ${\cal O}\big(\sqrt{T}\big)$ regret.
%\blue{relative to the most reliable edge server for each IoT device. }
	Information sharing among devices can accelerate the security-aware computing tasks. Incorporating the information shared by other devices, SAVE-S and SAVE-A offer impressive improvements on the sublinear regret, which is guaranteed by what is termed ``value of cooperation.'' Effectiveness of the proposed schemes is tested on both synthetic and real datasets.
\end{abstract}
\section{Introduction}\label{s.model}

Internet of Things (IoT) impact every aspect of daily life ranging from healthcare, video recognition in smart homes and smart cities, to monitoring smart grids \cite{chiang2016}. Among these, various latency-sensitive applications such as autonomous driving and virtual reality raise new challenges to the current IoT paradigms; e.g., a critical question is how to simultaneously meet the demands of huge volume and latency-sensitive data requests under limited computation power of IoT devices. 
Clearly, cloud computing is encountering growing challenges in meeting such new requirements in IoT.
Since the concept was first proposed in 2009 \cite{satyanarayanan2009}, edge computing has been viewed as a promising solution by subsiding the computing capability from cloud to edge servers to facilitate real-time computation services \cite{pan2018,wang2018,chen2016efficient,chiang2016,chen2017c}.  
%, which places computers with high computation power at strategic locations in order to provide both computation resources and storage for the users in vicinity.

Although edge computing enables offloading computationally intensive tasks to the edge, security issues could prevent one from fully embracing its advantages \cite{salameh2018,garnaev2016,ni2018,yang2017}. As an example, although edge computing facilitates location based-services (e.g., social networks),  users' location information is also exposed to edge nodes as well, from which a ``malicious'' edge node can pry into users' privacy \cite{ni2018}. In addition, data collected from critical infrastructure such as smart grids need also to be veiled in order to prevent blackouts and other malicious attacks. However, the high-complexity encryption techniques with complicated cipher-decipher processes are usually not computationally affordable for IoT devices (e.g., sensors in smart grids) \cite{ni2018}. One approach to coping with such privacy concerns is to allow devices to choose their trusted service on-demand with the so-called ``transparent computing model'' \cite{ren2017}, while the trustworthiness of edge servers could be evaluated by the trust management services \cite{noor2016}.

Besides privacy concerns, jamming and eavesdropping - the two main attacks at the physical layer \cite{zou2016,roman2018} -  are still an issue for IoT systems. This work mainly focuses on security-aware edge computing under jamming attacks that will block the communication link between IoT devices and edge servers; see \cite{hu2018} for schemes dealing with eavesdropping. 

Existing works dealing with jamming attacks can be categorized to those on anomaly detection \cite{lu2014} and jamming mitigation \cite{salameh2018,garnaev2016,oro2017,zhang2016,wang2012,chang2017,zhou2017}. 
Most incorporate extra resources such as power and spectrum to alleviate the effect of jamming. 
Optimal power allocation schemes for jamming attacks were studied in \cite{oro2017}. Assuming a low-power jammer, Bayesian game based anti-jamming strategies were reported in \cite{garnaev2016,zhang2016}.
However, for IoT devices with limited battery capacity, excess power consumption is not always affordable.
On the other hand, spectrum allocation to evade jamming was considered in \cite{salameh2018} for cognitive radio based IoT networks. Frequency hopping strategies without pre-shared secrets, a.k.a. uncoordinated frequency hopping, was adopted in \cite{wang2012}, where a multi-armed bandit (MAB) scheme was introduced to allocate frequency bands only considering one transmitter-receiver pair; and in \cite{chang2017} for large-scale and more sophisticated cognitive radio networks. Recent efforts were also devoted to improve energy efficiency via channel hopping based anti-jamming schemes. An MAB-based channel selection with an additional online gradient ascent to maximize energy efficiency was proposed in \cite{zhou2017}.
However, due to IoT devices operating over different frequency bands, spectrum expansion-based schemes have limited appeal especially for the spectrum-scarce IoT setups.

In this paper, we develop novel approaches to tackle the security-aware edge computing problem based on MAB with sleeping arms \cite{kanade2009,kleinberg2010,li2018}, where the number of accessible servers can be time-varying. 
Specifically, each IoT device progressively learns the risk associated with edge servers and adaptively chooses an edge server to offload computing tasks among all available servers per slot. 
In addition, we further consider the case where IoT devices can share information to cooperatively achieve security-aware edge computing. 
%However, the dynamically changing environment in jamming attacks also blocks the communication links between devices. 
To account for the mobility of IoT devices, we model the communication links as a time-varying directed graph. Under this model, performance gain of leveraging cooperation in securing edge computing is also rigorously established. 
%\blue{To characterize the value of cooperation, we define the cooperation value $\lambda$. Typically $\lambda \leq 1$ and a smaller $\lambda$ positively links with the significance of improvement via information sharing.}
%It worth mention that although graph based information sharing facilitates the anti-jamming tasks, to deal with the general cases with isolated devices, cooperation is not necessarily required.

Our main contributions are summarized as follows:
\vspace{-0.1cm}
\begin{enumerate}
	\item[\textbf{c1)}] Leveraging the MAB framework, we develop SAVE-S and SAVE-A algorithms to deal with edge computing under stochastic and adversarial jamming attacks, respectively. Incorporating the information shared by allied devices, SAVE-S and SAVE-A can accelerate the security-aware computing.
	\item[\textbf{c2)}] We analytically establish that an ${\cal O}\big(\sqrt{T}\big)$ regret can be achieved by both SAVE-S and SAVE-A, and benefit from device cooperation with markedly lower risks - a quantifiable performance gain of cooperation that we call the \emph{value of cooperation}.  
	\item[\textbf{c3)}] Simulation tests have been conducted using both synthetic and real data to showcase the effectiveness of the proposed schemes, and confirm the value of cooperation for security-aware computing.
\end{enumerate}

\noindent\textbf{Notations:} $\mathbb{E}$ denotes the expectation; $\mathds{1}$ denotes the indicator function; $(\cdot)^{\top}$ stands for vector transposition; $\mathbf{x}$, while inequalities for vectors are defined entry-wise.

\section{Models and Problem Statements}

This section introduces the models and the formulation of the security-aware problem under jamming attacks.

\subsection{Modeling preliminaries}
Consider an IoT network with a set of ${\cal K}:= \{1,2,\ldots, K\}$ edge servers to serve computational requestes for a set of IoT devices ${\cal J}:= \{1,2,\ldots, J\}$. Per slot $t$, the task of device $j$ is denoted by a tuple $(c_t^j,s_t^j)$ with $c_t^j$ denoting the required resources (e.g., CPU cycles) to complete the task, and $s_t^j$ the size of computation task (the data input and the associated processing code) \cite{chen2016efficient,chen2018task}.

\emph{Security risk.}
When an edge server is being attacked, it can behave unfaithfully or intentionally sabotage the computation tasks. Therefore, to alleviate such a compromise on privacy, an IoT device has to select the most reliable edge server for offloading computing tasks; see a diagram for security-aware edge computing in Fig. \ref{fig.app}. To measure the edge server's reliability, we rely on the security risk, which depends on e.g., the number of attacks within a slot duration \cite{noor2016}. Let $\gamma_{1,t}(k)$ denote the unit risk of computing at server $k$, and $\gamma_{2,t}(k)$ the unit risk for privacy information leakage at server $k$, measured as in \cite{noor2016}. 
Furthermore, let $\rho_j \in [0,1]$ denote a device-specific weight on security; e.g., a larger $\rho_j$ can be used for safety-sensitive tasks involved in autonomous vehicle and healthcare, while a smaller $\rho_j$ can be adopted by sensors in smart homes and smart grids where privacy is the first priority. 
Per slot $t$, the security risk for device $j$ choosing server $k$ is modeled by \cite{chen2017b}
\begin{equation}\label{eq.risk}
	r_t^j (k) = \rho_j c_t^j \gamma_{1,t} (k)  + (1-\rho_j) s_t^j \gamma_{2,t} (k).
\end{equation}
Note that the up-to-date security evaluation is available only after the IoT devices communicate with edge servers; that is, $\gamma_{1,t} (k)$ and $\gamma_{2,t} (k)$ are obtained at the end of slot $t$.

%\emph{Edge server capacity.}
%Comparing with relying on a remote cloud server, edge computing holds promises to reduce latency for IoT applications. To guarantee an approximate real-time processing, the tasks should be spread among edge servers in an IoT network. Mathematically, the accumulated tasks on server $k$ should not exceed the delay toleration bound $\beta^k$, e.g., the following constraint should be satisfied in long run:
%\begin{equation}\label{eq.constraint}
%	\frac{1}{T}\sum_{t=1}^{T}\sum_{j=1}^{J} c_t^j \mathds{1}(a_t^j = k) \leq \beta^k, ~\forall k \in {\cal K}.
%\end{equation}
%where $a_t^j \in {\cal K}, ~\forall j, \forall t$ denotes the selected edge server of device $j$ at slot $t$. Instead of a stringent requirement that \eqref{eq.constraint} is satisfied at each slot, the long-term constraint offers extra flexibility on server selection to ``smartly'' violate the constraint. For a succinct presentation, we assume $\beta^k$ is time invariant, but a generalization to time variant upper bound is straightforward.

\emph{Jamming attacks.}
Malicious attackers can also sabotage the IoT devices with jamming attacks, which block their access to edge servers \cite{zou2016,roman2018}. Jammers can be classified as \textit{stochastic jammers}, which attack the IoT devices with a fixed probability; or, as \textit{adversarial jammers} intelligently deciding which edge servers to block \cite{garnaev2016}. With the presence of jammers, collect the accessible servers for device $j$ in ${\cal K}_t^j \subseteq {\cal K}$, from which device $j$ selects the most reliable server. Due to the directionality of jamming, for two different devices $j$ and $j'$, it is possible that ${\cal K}_t^j \neq {\cal K}_t^{j'}$. Note that $r_t^j (k)$ for jammed server $k \notin {\cal K}_t^j$ is still well-defined, but server $k$ cannot be selected by device $j$ and thus $r_t^j (k)$ is not revealed. 

\emph{Device cooperation.}
Information sharing can assist the security-aware computation task at the edge. In this setup, information sharing takes places at the edge server as well as at the IoT devices side. Cooperation is forbidden at the edge server side, since it can expose serious privacy concerns by providing an extra medium for malware transportation \cite{cheng2017}. To expedite the server selection procedure, information sharing among devices can be beneficial. Per slot $t$, after device $j$ observes $\gamma_{1,t} (k)$ and $\gamma_{2,t} (k)$, it can communicate this information to its neighbors. Note that since the communication among devices demands extra energy, whether to participate in cooperation is device specific. Assume that information sharing is unidirectional, meaning that it is possible for $j'$ to receive information from $j$, but not vise versa. 
%On the other hand, due to the privacy concerns, sharing device-specific data such as $(c_t^j, s_t^j)$ are forbidden.

\begin{figure}[t]
	\vspace{-0.1cm}
	\centering
	\includegraphics[height=0.24\textwidth]{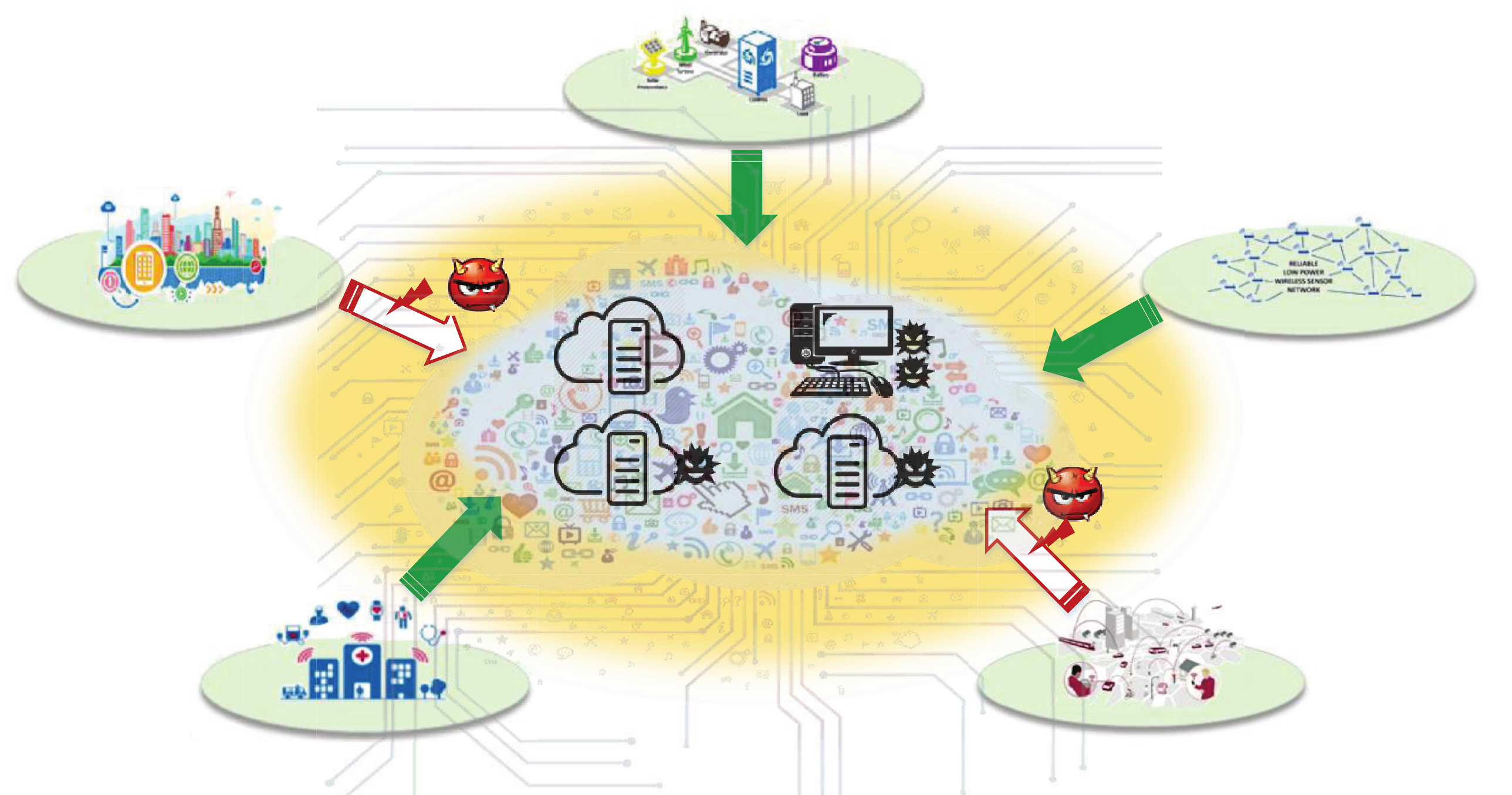}
%	\vspace{-0.20cm}
	\caption{Security-aware edge computing diagram.}
	\label{fig.app}
	\vspace{-0.3cm}
\end{figure}

In a nutshell, we wish to solve the following problem in an online fashion,
\begin{equation}\label{eq.prob1}
\min_{\{a_t^j \in{\cal K}_t^j, \forall t, \forall j\}}\,\sum_{t=1}^{T}\sum_{k=1}^{K}\sum_{j=1}^{J} r_t^j (k) \mathds{1}(a_t^j = k)
\end{equation}
where $a_t^j$ denotes the selected server by device $j$. In addition to solving \eqref{eq.prob1} with stochastic or adversarial jammers, the benefit of cooperation among devices is also of interest.
Note that when the present formulation consider a setting without explicit constraints, it can be readily extended to incorporate long-term constraints; see e.g., \cite{li2018,chen2017tsp}.

\subsection{Reformulation as online linear programming}

Since the security risk $r_t^j (a_t^j)$ is revealed afterwards, intuitively, the device should select servers according to the historical performance of servers, while allowing for flexibility to choose other servers. In par with this guideline, suppose device $j$ is allowed to randomly select a server from a given distribution $a_t^j\sim \mathbf{p}_t^j\in\mathds{R}^K$. Problem \eqref{eq.prob1} can be then reformulated as optimization over $\{\mathbf{p}_t^j\}_{t=1}^T$, namely
\begin{equation}\label{eq.prob2}
	\min_{\{\mathbf{p}_t^j\in{\Delta({\cal K}_t^j)},\forall t, j\}}\,\sum_{t=1}^{T}\sum_{j=1}^{J} \big(\mathbf{p}_t^{j}\big)^\top \mathbf{r}_t^j
\end{equation}
where the ${\cal K}_t^j$-related ``probability simplex'' is defined as
\begin{equation}
\Delta({\cal K}_t^j):=\left\{\mathbf{p} \in\mathds{R}_+^K\Bigg|\sum_{k\in {\cal K}_t} p(k)=1;\, p(k)\!=\!0, k\notin {\cal K}_t^j\right\}
\end{equation}
with $p(k)$ denoting the $k$-th entry of $\mathbf{p}$. 
%\red{The reason mentioned in blue part is suboptimal. Please check online learning literature how they motivate arm weights?} \green{I think here is to motivate $\mathbf{p}_t$ instead of $\mathbf{w}_t$?}
The logic behind the probability $\mathbf{p}_t^j$ in \eqref{eq.prob2} is that a randomized server selection scheme may have better performance in expectation than deterministic schemes in the worst case \cite[Theorem~1.1]{hazan2016}. For security considerations, a deterministic algorithm implicitly impairs security since adversaries can potentially infer the strategies in use given knowledge of the implemented algorithms. 
%Compared to the traditional MAB problem, 
The problem \eqref{eq.prob2} is simply a linear program, and thus can be readily solved had we known the sequence $\{\mathbf{r}_t^j\}_{t=1}^T$ as well as the accessible server sets $\{{\cal K}_t^j\}_{t=1}^T$. 
However, the challenge arises due to the causal knowledge of the security risks $\{\mathbf{r}_t^j\}_{t=1}^T$, and the difficulty here also comes from the need of adaptively choosing the edge server according to a time-varying feasibility set $\mathbf{p}_t^j\in{\Delta({\cal K}_t^j)}$. 
%To further boost the performance in large scale networks, cooperation among IoT devices will be adopted by leveraging the side information. 

%In the following, the problem \eqref{eq.disprob} is addressed leveraging the stochastic or adversarial nature of ${\cal K}_t^j$.

%Note that the problem is separable in IoT devices, i.e., minimizing \eqref{eq.prob2} is equivalent to solving the following problem for each device $j$
%\begin{equation}\label{eq.disprob}
%\min_{\{\mathbf{p}_t^j\in{\Delta({\cal K}_t^j)},\forall t\}}\,\sum_{t=1}^{T} \big(\mathbf{p}_t^{j}\big)^\top \mathbf{r}_t^j.
%\end{equation}

\section{Edge Computing under Stochastic Jamming}\label{sec.sto}
This section deals with edge server selection problem for stochastic jammers, where ${\cal K}_t^j$ follows a random process with an unknown but fixed distribution.

\subsection{Cooperation via a graph-encoded feedback}

The information sharing among devices can be modeled through directed graphs  \cite{kocak2014,kocak2016,alon2017}. 
Consider a single IoT device $j$, which at the end of slot $t$ obtains the security risk of the selected edge server $a_t^j$ as well as other servers' security risk (a.k.a. side observations) shared by other devices.

\begin{figure}[t]
	\centering
	\includegraphics[height=0.25\textwidth]{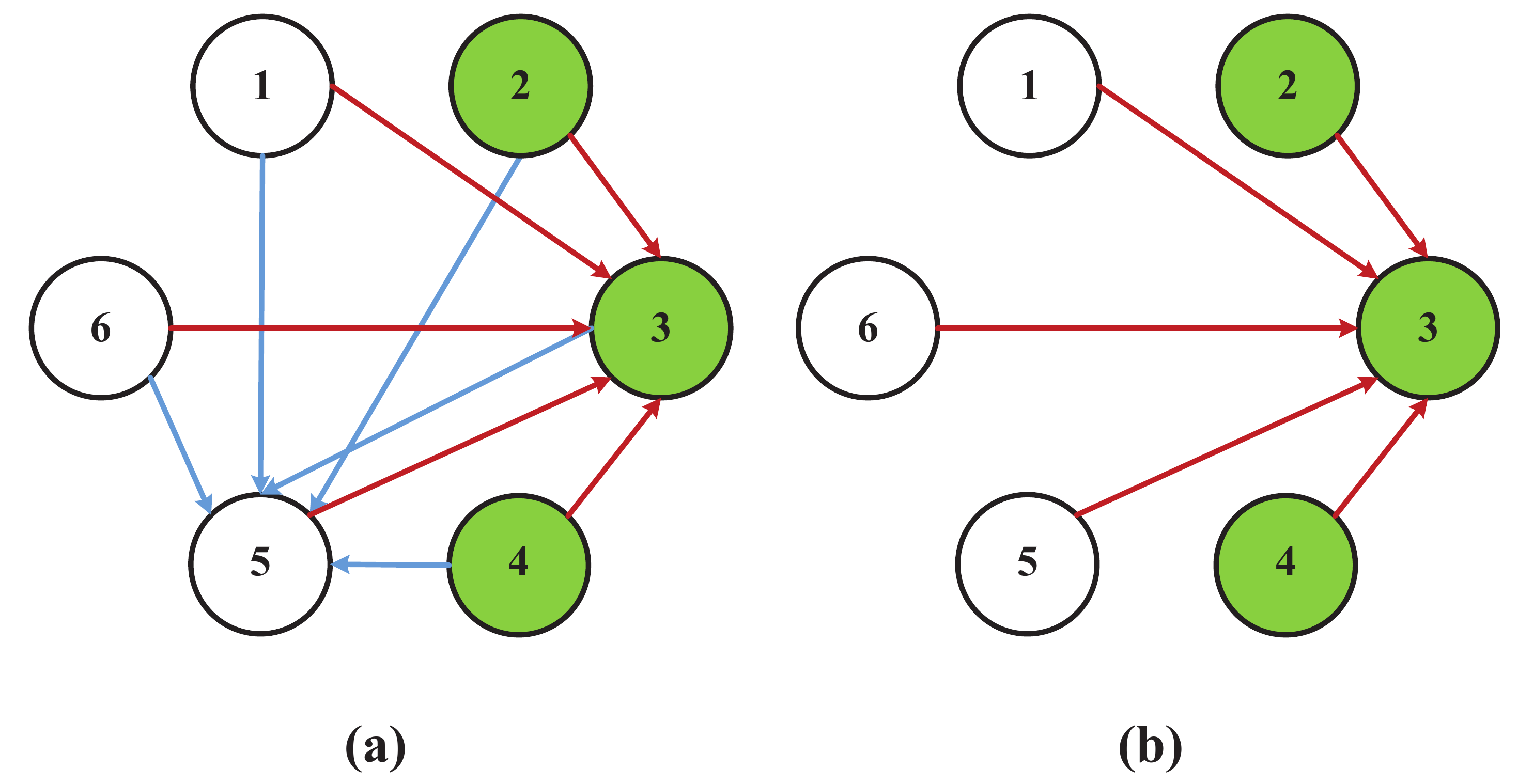}
	\vspace{-0.5cm}
	\caption{A side observation graph with node set ${\cal K} = \{1,\cdots, 6\}$.
		The available servers set for device $j$ is ${\cal K}_t^j = \{2,3,4\}$ with green nodes. The arrow pointed to each node itself is omitted for better visibility. (a) considers $a_t^j = 2$ and side observations are the security risks of server 3 and server 5. (b) considers the $a_t^j = 3$ and the only information shared is also about server 3.}
	\label{fig.soexample}
	\vspace{-0.2cm}
\end{figure}

To account for the mobility of IoT devices, and the asymmetry of communication links, we model the side observations for device $j$ via a time-varying directed graph ${\cal G}_t^j$. In this graph, the node sets are the server set $\cal K$; and the side observations form a subset of the node sets; thus, ${\cal S}_t^j \subseteq {\cal K}$. Note the ${\cal S}_t^j$ is not necessarily a subset of ${\cal K}_t^j$, meaning that it is possible for device $j$ to obtain the security risk of inaccessible edge servers via information sharing. The directed edge $(k_1,k_2)$ from node (server) $k_1$ to node (server) $k_2$ indicates that when device $j$ chooses server $k_1$, it can also observe (via information shared by other devices) the security risk associated with server $k_2$. For each node, there is an edge pointing to itself; see an example in Fig. \ref{fig.soexample}. The underlying graph associated with each device $j$ captures the observability relationship among servers. 

To facilitate the analysis and explicitly quantify the impact of information sharing, a few notations from graph theory are introduced \cite{alon2017}. An \textit{independent set} of an undirected graph is a set of vertices that are not connected by any edges; while the so-termed \textit{independence number} $\alpha_t^j$ is the cardinality of the maximum independent set. 
%For a directed graph, we simply view the edges as undirected and thus previous definition applies. 
For the example in Fig. \ref{fig.soexample} (a), we have $\alpha_t^j = 4$ with the largest independent set being $\{1,2,4,6\}$; while in (b) $\alpha_t^j = 5$, with independent set $\{1,2,4,5,6\}$. When the device $j$ does not receive any side observation, we have $\alpha_t^j = K$, which is also the upper bound on $\alpha_t^j$. 
Intuitively, per device $j$ at slot $t$, $\alpha_t^j$ reflects its knowledge of the global information on security risks, with a smaller value $\alpha_t^j$ indicating high connectivity of the information graph (more side observations), thus implying more secure computing performance.
%Although $\alpha_t^j$ is not explicitly used in proposed algorithm, it reflects how the shared information improves the learning procedure and facilitates the analysis.

\subsection{SAVE-S for stochastic jamming}
Given the underlying side observation graphs, we are ready to develop the proposed algorithm that we call Security-Aware edge serVer sElection under stochastic jamming (SAVE-S). The SAVE-S algorithm leverages side information shared by other devices to assist the learning procedure, while also accounting for the heterogeneity of devices via adaptive stepsizes for improved performance.

Aiming to solve \eqref{eq.prob2} for each device, per slot $t$, SAVE-S first evaluates each server's historical performance via
\begin{equation}\label{eq.weight}
w_t^j(k) = \exp\Big(-\eta_t^j \hat{R}_{t-1}^j (k)\Big), ~ \forall k \in {\cal K}
\end{equation}
where $\eta_t^j$ plays the role of a stepsize, and $\hat{R}_{t-1}^j (k)$ is the accumulated security risk, given by
\begin{equation}
	\hat{R}_{t-1}^j (k) = \sum_{\tau =1}^{t-1} \hat{r}_{\tau}^j(k), ~\forall k \in {\cal K}
\end{equation}
where $\hat{r}_{t}^j (k)$ is the estimated security risk, which will be formally defined soon. Then after the server set ${\cal K}_t^j$ is revealed, device $j$ finds the server selection probability as
\begin{equation}\label{eq.prob}
p_t^j(k) = \frac{w_t^j(k) \mathds{1}\big( k \in {\cal K}_t \big) }{ \sum_{m \in {\cal K}_t^j } w_t^j(m)}.
\end{equation}
It is clear that for servers $k \notin {\cal K}_t^j$, we have $p_t^j(k) = 0$. With $\mathbf{p}_t^j$, a server $a_t^j$ is chosen, and the corresponding security risk is revealed after the edge computing task is finished. Information sharing among IoT devices is then effected, from which device $j$ is informed about the security risks of servers ${\cal S}_t^j \subseteq {\cal K}$. Note that it is possible to have $a_t^j \in {\cal S}_t^j$.
%\red{ Different from \cite{alon2017}, where the server indices of ${\cal S}_t^{jk}$ should be known before the decision making, here all information of ${\cal S}_t^{jk}$ is revealed afterwards.} 
Usually in MAB algorithms, an unbiased estimator is usually adopted \cite{auer2002a,hazan2016,bubeck2012}. Differently, here we adopt a biased but reduced-variance estimator \cite{kocak2014}, given by
%\begin{subequations}
%	\begin{equation}\label{eq.est1}
%	\hat{r}_t^j (k) = \frac{r_t^j (k) \mathds{1}\big( k = a_t^{jk} \big) }{p_t^j(k) + \mu_t^j}, ~\forall k \notin {\cal S}_t^j;
%	\end{equation}
%	\begin{equation}\label{eq.est2}
%	\hat{r}_t^j (k) =\frac{ r_t^j (k)}{1+\mu_t^j}, ~\forall k \in {\cal S}_t^j;
%	\end{equation}
%\end{subequations}	
\begin{equation}\label{eq.est}
\hat{r}_t^j (k) =\left\lbrace \begin{array}{ll}
 	 \frac{r_t^j (k) \mathds{1}\big( k = a_t^j \big) }{\mu_t^j + p_t^j(k)}, ~&\forall k \notin {\cal S}_t^j\\
 	~~~~\frac{ r_t^j (k)}{\mu_t^j+1}, ~&\forall k \in {\cal S}_t^j
 \end{array}\!\right.
	\end{equation}	
which can be compactly written as
\begin{equation}\label{eq.est3}
\hat{r}_t^j (k) =\frac{ r_t^j (k) \mathds{1} \big( k \in \{ a_t^j \cup {\cal S}_t^j\} \big)}{\mu_t^j + \sum_{(m, k) \in {\cal G}_t^j} p_t^j(m)}, ~\forall k \in {\cal K} 
\end{equation}
where $(m, k) \in {\cal G}_t^j$ denotes all the edges pointing from node $m$ to node $k$ in the underlying graph ${\cal G}_t^j$. 
The rationale behind the biased estimators \eqref{eq.est} stems from well known bias-variance tradeoff encountered with MAB. It is clear from \eqref{eq.est} that since $\mu_t^j > 0$, the security risk is always underestimated
\begin{equation}
\mathbb{E}\big[ \hat{r}_{t}^{j}(k) \big] =  \frac{\bigg(\sum_{(m, k) \in {\cal G}_t^j} p_t^j(m)\bigg)r_t^j (k)}{ \mu_t^j+\sum_{(m, k) \in {\cal G}_t^j} p_t^j(m)} < r_t^j (k).
\end{equation}
For the squared mean, we have
\begin{align}
\!\!\mathbb{E}\Big[ \big( \hat{r}_{t}^{j} (k) \big)^2  \Big] &= \! \frac{ \big(r_t^j (k) \big)^2 \sum_{(m, k) \in {\cal G}_t^j} p_t^j(m) }{\Big( \mu_t^j\!+\!\sum_{(m, k) \in {\cal G}_t^j} p_t^j(m) \Big)^2}\nonumber\\
&<\! \frac{\big(r_t^j (k) \big)^2}{ \mu_t^j\!+\!\sum_{(m, k) \in {\cal G}_t^j} p_t^j(m)}.
\end{align}
Comparing with EXP3 in \cite{auer2002a,bubeck2012}, the mean-square of the SAVE-S estimator is smaller due to the presence of $\mu_t^j$. The SAVE-S algorithm is summarized in Alg. \ref{algo1}.

%%%%%%%%%%%%%%%%%%%%%%%%%%%%%%%%%%%%%%%%%%%%%%%%%%%%
\begin{algorithm}[t]
	\caption{SAVE-S for IoT device $j$}\label{algo1}
	\begin{algorithmic}[1]
		\State \textbf{Initialize:}
		weight $\mathbf{w}_1^j = \mathbf{1}/K$, implicit explore factor $\mu_t^j$, learning rate $\eta_t^j$.
		\For {$t=1,2,\dots, T$}
		\State Compute $\mathbf{w}_t^j$ via \eqref{eq.weight}.
		\State Available server set ${\cal K}_t^j$ is revealed.
		\State Compute $\mathbf{p}_t^j$ via \eqref{eq.prob} and choose server $a_t^j \sim \mathbf{p}_t^j$.
		\State Receive $\gamma_{1,t}(a_t^j)$ and $\gamma_{2,t}(a_t^j)$. 
		\State Broadcast $\!\gamma_{1,t}(a_t^j)\!$ and $\!\gamma_{2,t} (a_t^j)\!$ to devices in $\{i\,|\, j \!\in\!{\cal S}_t^i\}$.
		\State Compute security risk for $\{a_t^j\} \cup {\cal S}_t^j$ via \eqref{eq.risk}.
		\State Estimate the security risk via \eqref{eq.est}.
		\EndFor
	\end{algorithmic}
\end{algorithm}
%%%%%%%%%%%%%%%%%%%%%%%%%%%%%%%%%%%%%%%%%%%%%%%%%%%%

\subsection{Regret analysis}
To evaluate the performance of SAVE-S, the first step is to seek a suitable benchmark for this setting. Unlike the best fixed server or server distribution in a standard MAB setup \cite{hazan2016,bubeck2012,auer2002a}, here the benchmark accounting for jammed servers is selected to be the best fixed server list \cite{kanade2009,kleinberg2010}.
To concretely define this benchmark, we first introduce the notion of a server list $\phi^j\in\mathds{N}^K$ for device $j$ that is a permutation of server indices. Then let $\Phi^j({\cal K}_t^j)$ represent the highest ranked server in list $\phi^j$ that is also in ${\cal K}_t^j$; e.g., for $K=3$ servers, if $\phi^j:=\{2,3,1\}$, and ${\cal K}_t^j=\{1,3\}$, then we have $\Phi^j({\cal K}_t^j)=3$. Furthermore, let $\phi^{j*}$ denote the best server list satisfying
\begin{equation}\label{eq.permutaion}
\sum_{t=1}^{T}r_t^j\big(\phi^{j*}(1)\big)\leq \cdots \leq \sum_{t=1}^{T}r_t^j\big(\phi^{j*}(K)\big)
\end{equation}
where $\phi^{j*}(i)$ is the $i$-th ranked server in $\phi^{j*}$. Correspondingly, define $\Phi^{j*}\big({\cal K}_t^j\big)$, which maps ${\cal K}_t^j$ to the highest ranked index in server list  $\phi^{j*}$. Then device $j$ incurs regret
\begin{equation}\label{eq.sreg}
	{\rm Reg}_T^j := \sum_{t=1}^{T} \mathbb{E}\big[r_t^j(a_t^j)\big] - \sum_{t=1}^{T} r_t^j \big(\Phi^{j*}({\cal K}_t^j)\big).
\end{equation} 
The rationale behind this regret definition is that instead of sticking to the best server in hindsight which might be unaccessible, we resort to the best available server in $\phi^{j*}$. The following assumption is used in the subsequent analysis. 

\vspace{0.2cm}
\noindent\textbf{(as1)} \textit{The security risk satisfies $\max_{t,j,k} r_t^j (k) \leq 1$.}
\vspace{0.2cm}

Basically, (as1) requires the security risk to be bounded, which is typical for online learning \cite{chen2017tsp,bubeck2012,hazan2016,auer2002a}. Building on (as1), the following theorem establishes the regret bound.

%The analysis starts with defining a graph related auxiliary variable $Q_t^j$ for each ${\cal G}_t^j$, namely,
%\begin{equation}
%	Q_t^j:= \sum_{k=1}^{K}\frac{p_t^j(k)}{\mu_t^j+\sum_{(m, k) \in {\cal G}_t^j} p_t^j(m)}.
%\end{equation}
%Note that if $k \in {\cal S}_t^j$, we have ${p_t^j(k)}/\big(\mu_t^j+\sum_{(m, k) \in {\cal G}_t^j} p_t^j(m) \big) = p_t^j(k)/(1+\mu_t^j)$, while $k \notin {\cal S}_t^j$, we have  
%${p_t^j(k)}/\big(\mu_t^j+\sum_{(m, k) \in {\cal G}_t^j} p_t^j(m) \big) = p_t^j(k)/(p_t^j(k)+\mu_t^j)$. The following lemma renders useful upper and lower bounds on $Q_t^j$.
%\begin{lemma}\label{lemma.1}
%	If $\mu_t^j \leq 1$ for every $t$, $Q_t^j$ is bounded by
%	\begin{align}
%		\frac{1}{1+\mu_t^j}\leq Q_t^j \leq \alpha_t^j  + \sum_{k \in {\cal S}_t^j}  p_t^j(k) - \sum_{k \in {\cal S}_t^j}\frac{\mu_t^j p_t^j(k)}{2},
%	\end{align}
%	where $\alpha_t^j$ is the independence number of ${\cal G}_t^j$.
%\end{lemma}
%\begin{proof}
%	See Appendix A.
%\end{proof}

%As a warm start, Lemma \ref{lemma.2} states that when the active server set is time-invariant, the regret is bounded by \eqref{eq.lemma2}. 
%Based on which the next theorem follows by linking Lemma \ref{lemma.2} to time-varying server set and thus the nontrivial regret in \eqref{eq.sreg}.

\begin{theorem}\label{theo.1}
	For stochastically chosen ${\cal K}_t^j$, the regret for device $j$ can be bounded by
	\begin{equation}\label{eq.theo1}
	\mathbb{E} \big[{\rm Reg}_T^j \big] \leq \mathbb{E} \bigg[\sum_{t=1}^{T} \bigg( \mu_t^j+ \frac{\eta_t^j}{2} \bigg) Q_t^j + \frac{\ln K}{\eta_{T+1}} \bigg]
	\end{equation}
	where the expectation is over the randomness of ${\cal K}_t^j$, and the auxiliary variable $Q_t^j$ is
	\begin{equation}\label{eq.def-q}
		Q_t^j:= \sum_{k=1}^{K}\frac{p_t^j(k)}{\mu_t^j+\sum_{(m, k) \in {\cal G}_t^j} p_t^j(m)}.
	\end{equation}
\end{theorem}
\begin{proof}
	See Appendix \ref{appendix.stheo}.
\end{proof}

 The bound on $\mathbb{E} \big[{\rm Reg}_T^j \big]$ for SAVE-S depends on an auxiliary variable $Q_t^j$ capturing the influence of side observations as well as the biased $\hat{r}_t^j (k)$. 
 For example, in EXP3 (with $\mu_t^j = 0$ and no side information), we have $Q_t^j = \sum_{k=1}^K p_t^j(k)/p_t^j(k) = K$. 
For SAVE-S, it is shown in Lemma \ref{lemma.1} of Appendix \ref{appendix.sQ} that $Q_t^j$ is bounded by $\min \{ K, \alpha_t^j+1\}$. Intuitively, a small $\alpha_t^j$ associated with a closely connected graph ${\cal G}_t^j$ will lead to lower regret.

To evaluate the performance gain of cooperation, we revisit the performance without information sharing among devices.
%The simplest method is to choose fixed or diminishing stepsizes.
\begin{corollary}\label{col.fixstepsize}
	Consider the case without device cooperation, e.g., skip line 7 in Alg. \ref{algo1}. If the stepsizes are chosen as $\eta_t^j =\sqrt{{\ln K}/(KT)}$ and $\mu_t^j = \frac{\eta_t^j}{2}, \forall t$, the regret is bounded by
	\begin{equation}\label{eq.fixreg1}
		\mathbb{E}\big[{\rm Reg}_T^j\big] \leq 2\sqrt{TK\ln K}
	\end{equation}
	 where the expectation is w.r.t. the randomness of the servers' availability. If $\eta_t^j =\sqrt{\frac{\ln K}{2Kt}}, ~\forall t$ and $\mu_t^j = \frac{\eta_t^j}{2}, \forall t$, we have 
	 \begin{equation}\label{eq.fixreg2}
	 \mathbb{E}\big[{\rm Reg}_T^j\big] \leq 2\sqrt{2TK\ln K}.
	 \end{equation}
\end{corollary}
\begin{proof}
	See Appendix \ref{appendix.sfdstep}.
\end{proof}

Corollary \ref{col.fixstepsize} asserts that if we choose fixed $\eta_t^j$ and $\mu_t^j$, an ${\cal O}(\sqrt{TK \ln K})$ regret is guaranteed even without cooperation, which matches that of \cite{kanade2009,kleinberg2010}. Instead, if we choose diminishing $\eta_t^j$ and $\mu_t^j$, a slightly worse bound can be achieved. Then in the following corollary, we present a tighter bound by adopting the underlying graph introduced earlier.  
\begin{corollary}\label{col.adastepsize}
	If the devices cooperate and we choose stepsizes adaptively, namely, $\eta_t^j = \sqrt{ (\ln K) / \big( K+ \sum_{\tau = 1}^{t-1}Q_\tau^j\big)}$, and $\mu_t^j =\eta_t^j/2, \forall t$, the regret can be bounded as
	\begin{equation}\label{eq.adareg}
	\mathbb{E}\big[{\rm Reg}_T^j\big]\!\leq 2\mathbb{E}\Bigg[\sqrt{ \Big(\delta + \sum_{t=1}^{T} Q_t^j \Big) \ln K}\Bigg]\!\!\approx 2\mathbb{E}\Bigg[\sqrt{ \sum_{t=1}^{T} Q_t^j \ln K} \Bigg]
	\end{equation}
	where $\delta := \min_t \{ K- Q_t^j \}$, and the expectation is w.r.t. the randomness of the servers' availability.
\end{corollary}
\begin{proof}
	See Appendix \ref{appendix.sadastep}.
\end{proof}

The stepsizes $\eta_t^j$ and $\mu_t^j$ depend only on the history, regardless of the current $Q_t^j$. At first glance, it is clear that the bound in \eqref{eq.adareg} is better than \eqref{eq.fixreg1} and \eqref{eq.fixreg2} since $\sum_{t=1}^{T} Q_t^j \leq KT$. 
To characterize the value of cooperation, define the cooperation value $\lambda^j$ for device $j$ as the regret bound in \eqref{eq.adareg} divided by that in \eqref{eq.fixreg1}, which quantifies the improvement of leveraging cooperation among devices. Ideally, $\lambda^j \leq 1$ suggests that the cooperation reduces the security risk in the worst case. 
In addition, the cooperation value for this IoT network is the average of $\lambda^j$, namely
\begin{equation}
	\lambda = \frac{1}{J}\sum_{j=1}^J\lambda^j := \frac{1}{J}\sum_{j=1}^J\frac{\mathbb{E}\Bigg[\sqrt{ \Big(\delta + \sum_{t=1}^{T} Q_t^j \Big) \ln K}\Bigg]}{\sqrt{TK\ln K} }.
\end{equation}

The upper bound of the cooperation value $\lambda$ is given in the following corollary.
\begin{corollary}
	The cooperation value of SAVE-S satisfies
	\begin{equation}\label{eq.scoop}
		\lambda \leq \frac{1}{J}\sum_{j=1}^J \sqrt{\frac{1}{T}+ \frac{1}{KT}\sum_{t=1}^T \min \left\{K, K+1 - |{\cal S}_t^j| \right\}}.
	\end{equation}
\end{corollary}
\begin{proof}
	See Appendix \ref{appendix.sQ}.
\end{proof}

The bound in \eqref{eq.scoop} asserts that more side observations reduce the regret. Specifically, suppose $K/2$ servers' information can be obtained by device $j$ via information sharing. In this case, we have $|{\cal S}_t^j|= \frac{K}{2}$, which leads to $\lambda^j \leq \sqrt{\frac{1}{T}+ \frac{1}{2}+\frac{1}{K}}$.

The multiple choices of stepsizes are also tailored for the heterogeneity of IoT devices. For those isolated devices, fixed $\mu_t^j$ and $\eta_t^j$ could be adopted; and for those devices with sufficient side observations, adaptive stepsizes are recommended for improved performance.

\section{Edge Computing under Adversarial Jamming}
The stochastic jamming considered so far is relatively simple to deal with since the accessible servers are invariant in expectation, even if their realization is time-varying. In this section, we introduce schemes for adversarial jammers.

\subsection{From server to server list selection}
With insights gained from \cite{kanade2009} and \cite{kleinberg2010}, the high-level idea is to deal with the available server set in the presence of adversaries, by expanding the search space from the best single server to the best server ordering (a.k.a. server list). 
%The server list trick was also advocated in \cite{kanade2009,kleinberg2010}. 
Recall that a server list $\phi\in\mathds{N}^K$ is a permutation of server indices.
Let set $\check{\cal K}$ collect all the permutations of ${\cal K}$ with cardinality $\check{K}:=|\check{\cal K}| = K!$. 
With $\phi_t^j$ denoting the server list selected by IoT device $j$ at slot $t$ and the current available set ${\cal K}_t^j$, $\Phi_t^j({\cal K}_t^j)$ represents the highest ranked server in the list $\phi_t^j$ that is also in ${\cal K}_t^j$, see Fig. \ref{fig.listsoeg} (a) for example. Given the available server set ${\cal K}_t^j$ for device $j$, the mapping $\Phi_t^j(\cdot)$ plays the role of a policy at time $t$, which outputs the arm $a_t^j$.  
Therefore, if we use a $\check{K}\times K$ matrix $\bm{\Gamma}({\cal K}_t^j)$ to represent all these mappings at time $t$, the $(i, k)$th entry of $\bm{\Gamma}({\cal K}_t^j)$ can be written as 
\begin{equation}\label{eq.GammaDef}
	\left[\bm{\Gamma}({\cal K}_t^j)\right]_{i,k}=\mathds{1}(\Phi_i^j({\cal K}_t^j)=k).
\end{equation}
Accordingly, for adversarial jammers, \eqref{eq.prob2} can be rewritten as the following problem over distributions $\{\mathbf{q}_t^j\}$, namely
\begin{align}\label{eq.prob3}
	\min_{\{\mathbf{q}_t^j\in{\Delta^{\check{K}}},\forall j\}}\,\sum_{j=1}^J\sum_{t=1}^{T} \big(\mathbf{q}_t^j\big)^\top \check{\mathbf{r}}_t^j
\end{align}
where $\check{\mathbf{r}}_t^j$ is a $\check{K}$-dimensional vector, $\check{\mathbf{r}}_t^j = \bm{\Gamma}({\cal K}_t^j) \mathbf{r}_t^j$; and the $\check{\cal K}$-dimensional probability simplex is defined as
\begin{equation}
	\Delta^{\check{K}}:=\left\{\mathbf{q}\in\mathds{R}_+^{\check{K}}\Bigg|\sum_{k\in \check{\cal K}} q(k)=1 \right\}.	
\end{equation}
Instead of finding a $K$-dimensional vector $\mathbf{p}^j$ to weigh all the servers, our solution here is to search for a $\check{K}$-dimensional vector $\mathbf{q}^j$ that weighs all the server lists in $\check{\cal K}$. The following lemma establishes that \eqref{eq.prob2} is equivalent to \eqref{eq.prob3}. 
\begin{lemma}\label{lemma.reformulation}
	For each $\mathbf{p}_t^j$, there exists at least one $\mathbf{q}_t^j \in \Delta^{\check{K}}$, such that $\big(\mathbf{q}_t^j \big)^\top \check{\mathbf{r}}_t^j = \big(\mathbf{p}_t^j\big)^\top \mathbf{r}_t^j$.
\end{lemma}
\begin{proof}
	See Appendix \ref{appendix.sreform}.
\end{proof}

Another advantage of expanding the search space to a server list is reflected on the regret analysis. The benchmark to compare with in \eqref{eq.sreg} is exactly the best server list $\phi^{j*}$, which is equivalent to finding $\mathbf{q}^{j*} = [0,\ldots, 1,\ldots,0]^\top$, namely,
\begin{equation}\label{eq.phistar}
	\Phi^{j*}({\cal K}_t^j) =\big(\mathbf{q}^{j*}\big)^\top\bm{\Gamma}({\cal K}_t^j) \mathbf{r}_t^j = \big(\mathbf{q}^{j*}\big)^\top\check{\mathbf{r}}_t^j.
\end{equation}
The regret in \eqref{eq.sreg} can be rewritten as [cf. \eqref{eq.phistar}]
\begin{equation}\label{eq.regreformulate}
	{\rm Reg}_T^j = \sum_{t= 1}^{T} \big( \mathbf{q}_t^j\big)^\top \check{\mathbf{r}}_t^j -  \big(\mathbf{q}^{j*}\big)^\top\check{\mathbf{r}}_t^j
\end{equation}
which will further facilitate the analysis.

\subsection{SAVE-A for adversarial jamming}
For \eqref{eq.prob3}, one can still implement EXP3 \cite{auer2002a} over the expanded search space $\Delta^{\check{K}}$, but its efficiency can be significantly improved with cooperation among devices. 
Tailored for this setting, we will develop next a Security-Aware edge serVer sElection under adversarial jammer (SAVE-A) algorithm.

To address the challenge, we can again rely on the graph-encoded feedback structure. However, corresponding to the enlarged search space, the node set of ${\cal G}_t^j$ comprises all server lists and thus has cardinality $\check{K}$. Different from the stochastic setup, instead of selecting a signle server directly, we are seeking the best server list, where cooperation plays a more important role. That is to say, in addition to adding new edges in the underlying graph, the side observations change the structure of the original graph. With reference to Fig. \ref{fig.listsoeg}, the side observation ${\cal S}_t^j = \{1\}$ changes the original graph in two ways: i) it changes the structure by creating a virtual available set $\tilde{\cal K}_t^j:= {\cal K}_t^j \cup {\cal S}_t^j$; although the original available set is ${\cal K}_t^j = \{2,3\}$, with side observation ${\cal S}_t^j = \{1\}$, some nodes are enabled to have $\Phi_t^j \big(\tilde{{\cal K}}_t^j\big) = \{1\}$; and, ii) it introduces extra edges from each node (server list) to server lists with $\Phi_t^j\big(\tilde{{\cal K}}_t^j\big) = \{1\}$.

\begin{figure}[t]
	\vspace{-0.2cm}
	\centering
	\includegraphics[height=0.25\textwidth]{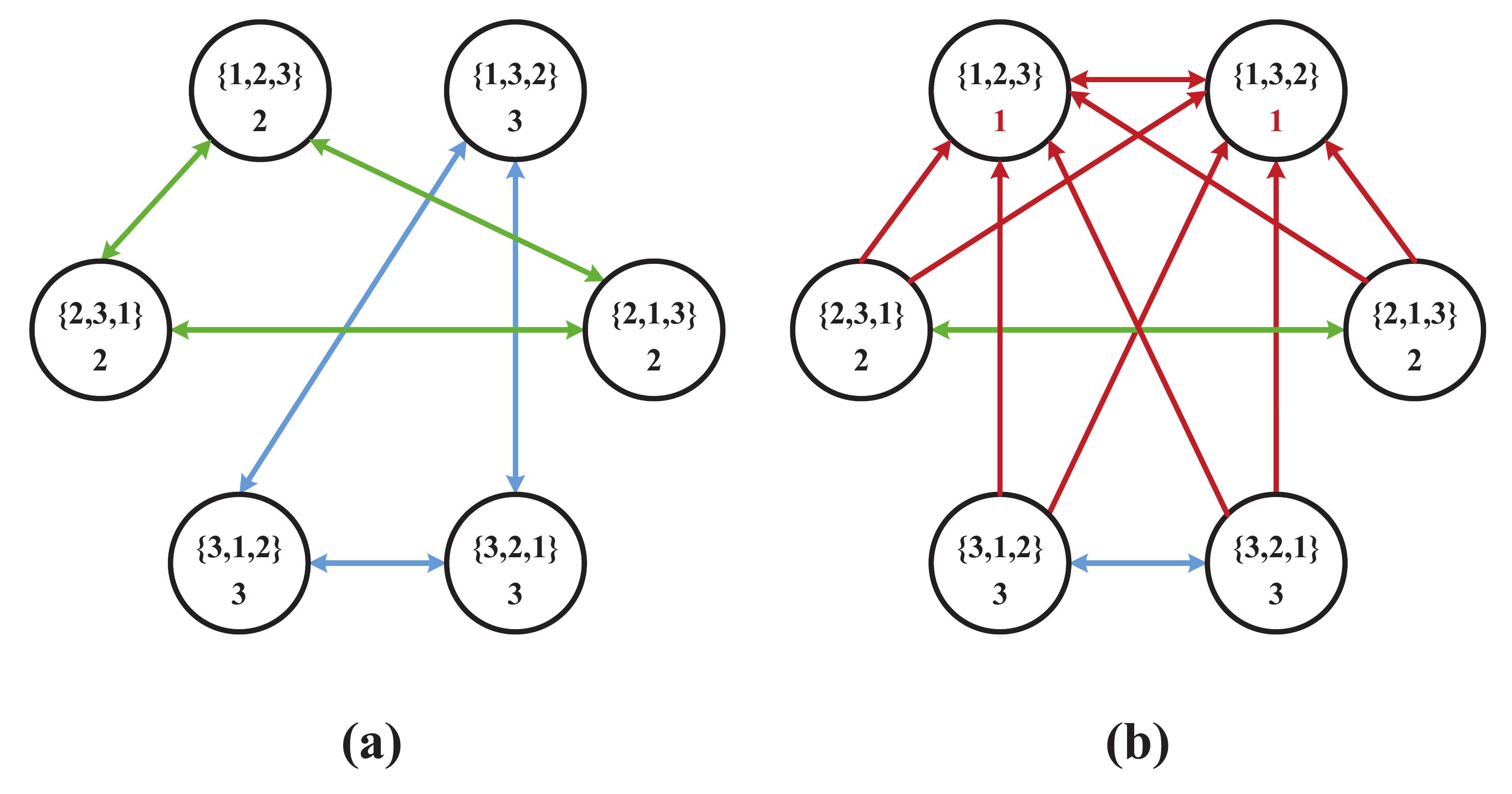}
	\vspace{-0.3cm}
	\caption{ A graph representation of side observation with each node denoting a server list. 	
	(a) considers ${\cal K}= \{1,2,3\}$, with ${\cal K}_t^j= \{2,3\}$ and no side observation. The directed arrows from node (server list) $\phi_t^{jm}$ to $\phi_t^{jn}$ denotes that when selecting list $\phi_t^{jm}$, the security risk of $\phi_t^{jn}$ is also revealed. The green lines connect server lists with $\Phi_t^j({\cal K}_t) = 2$, while $\Phi_t^j({\cal K}_t) = 3$ are linked through the blue lines. (b) considers active set ${\cal K}_t^j = \{2,3\}$, but with side observation ${\cal S}_t^j = \{1\}$. The red lines are the changes compared with (a).}
	\label{fig.listsoeg}
	\vspace{-0.3cm}
\end{figure}

Building upon the side information graph ${\cal G}_t^j$, an approach is developed next, in the search space $\Delta^{\check{K}}$. Specifically, IoT device $j$ maintains a weight $\mathbf{w}_t^j \in \mathds{R}^{\check{K}}$. Per slot $t$, based on the historical security risks, evaluation on each server lists' performances is carried out via
\begin{equation}\label{eq.aweight}
w_t^j(k) = \exp\Big(-\eta_t^j \hat{R}_{t-1}^j (k)\Big), ~ \forall k \in \check{\cal K}
\end{equation}
where the estimated cumulative risk $\hat{R}_{t-1}^j (k)$ is given by
\begin{equation}
\hat{R}_{t-1}^j (k) = \sum_{\tau =1}^{t-1} \hat{r}_{\tau}^j(k), ~\forall k \in \check{\cal K}
\end{equation}
with $\hat{r}_{t}^j (k)$ being the estimated security risk encountered with by server list $k$ at slot $t$. Then after the available set ${\cal K}_t^j$ is revealed, device $j$ computes the probability of selecting server list $k \in \check{\cal K}$ by normalizing $w_t^j(k)$; that is
\begin{equation}\label{eq.aprob}
q_t^j(k) = \frac{w_t^j(k) }{ \sum_{m \in \check{\cal K} } w_t^j(m)}.
\end{equation}
Then a server list $\phi^j_t$ is chosen according to the distribution $\mathbf{q}_t^j$, followed by the selection of an edge server $a_t^j = \Phi_t^j \big({\cal K}_t^j\big)$. After the computation tasks are finished, device $j$ observes the security risk of $a_t^j$ as well as ${\cal S}_t^j$,  upon which the side observation graph ${\cal G}_t^j$ is constructed. Then estimators of $\check{\mathbf{r}}_t^j$ are formed as
%\begin{subequations}
%	\begin{equation}\label{eq.aest1}
%	\hat{r}_t^j (k) =\frac{ r_t^j (k)  \mathds{1} \big( a_t^j = k \big) }{ \sum_{(m, k) \in {\cal G}_t^j} q_t^j(m) + \mu_t}, ~\forall k=\Phi_t^j\big(\tilde{\cal K}_t^j\big) \notin {\cal S}_t^j
%	\end{equation}
%	\begin{equation}\label{eq.aest2}
%	\hat{r}_t^j (k) = \frac{r_t^j (k) }{1 + \mu_t}, ~\forall k=\Phi_t^j \big(\tilde{\cal K}_t^j\big) \in {\cal S}_t^j
%	\end{equation}	
%\end{subequations}
\begin{equation}\label{eq.aest}
\hat{r}_t^j (k) =\left\lbrace \begin{array}{ll}
 	\frac{ \check{r}_t^j (k)  \mathds{1} \big( a_t^j = k \big) }{ \mu_t^j + \sum_{(m, k) \in {\cal G}_t^j} q_t^j(m)}, ~&\forall k=\Phi_t^j\big(\tilde{\cal K}_t^j\big) \notin {\cal S}_t^j\\
 	~~~~~~~~\frac{\check{r}_t^j (k) }{ \mu_t^j + 1}, ~&\forall k=\Phi_t^j \big(\tilde{\cal K}_t^j\big) \in {\cal S}_t^j
 \end{array}\!\right.
	\end{equation}	
which are then adopted to update $\hat{R}_t^j (k)$. The proposed SAVE-A approach is summarized in Algorithm \ref{algo2}.

%%%%%%%%%%%%%%%%%%%%%%%%%%%%%%%%%%%%%%%%%%%%%%%%%%%%
\begin{algorithm}[t]
	\caption{SAVE-A for IoT device $j$}\label{algo2}
	\begin{algorithmic}[1]
		\State \textbf{Initialize:}
		weight $\mathbf{w}_1^j = \mathbf{1}/\check{K}$, exploration factor $\mu_t^j$, and learning rate $\eta_t^j$.
		\For {$t=1,2,\dots, T$}
		\State Compute $\mathbf{w}_t^j$ via \eqref{eq.aweight}.
		\State Available server set ${\cal K}_t^j$ is revealed.
		\State Compute $\mathbf{q}_t^j$ via \eqref{eq.aprob} and choose server list $\phi_t^j \sim \mathbf{q}_t^j$.
		\State Select server $\!a_t^j \!\!=\!\! \Phi_t^j\big({\cal K}_t^j\big)$, receive $\!\gamma_{1,t} (a_t^j)\!$ and $\!\gamma_{2,t} (a_t^j)$.
		%\State Information sharing between IoT devices.
		\State Broadcast $\!\gamma_{1,t}(a_t^j)\!$ and $\!\gamma_{2,t} (a_t^j)\!$ to devices in $\!\{i\,|\, j\!\in\!{\cal S}_t^i\}$.
		\State Compute security risk for $\{a_t^j\} \cup {\cal S}_t^j$ via \eqref{eq.risk}.
		\State Estimate risk via \eqref{eq.aest}.
		\EndFor
	\end{algorithmic}
\end{algorithm}
%%%%%%%%%%%%%%%%%%%%%%%%%%%%%%%%%%%%%%%%%%%%%%%%%%%%

\begin{remark}
	To reduce the computation and memory complexity of SAVE-A, one pertinent idea is to leverage recent advances on function approximation to represent $\bm{\Gamma}(\cdot)$ via the low-dimensional random basis functions ``on-the-fly'' \cite{shen2018}. 
\end{remark}

\subsection{Performance analysis}
In this subsection, we analytically assess the performance of SAVE-A.

\begin{theorem}\label{theo.2}
	For adversarially chosen ${\cal K}_t^j$, the regret of SAVE-A in \eqref{eq.regreformulate} can be bounded by
	\begin{equation}
	{\rm Reg}_T^j \leq\sum_{t=1}^{T} \bigg( \mu_t^j+ \frac{\eta_t^j}{2} \bigg) Q_t^j + \frac{\ln \check{K}}{\eta_{T+1}}
	\end{equation}
	where $Q_t^j$ is defined as
	\begin{equation}\label{eq.def-q2}
		Q_t^j:= \sum_{k=1}^{\check{K}}\frac{q_t^j(k)}{\mu_t^j+\sum_{(m, k) \in {\cal G}_t^j} q_t^j(m)}.
	\end{equation} 
\end{theorem}
\begin{proof}
	The proof is similar to that of Theorem \ref{theo.1}, and it is thus omitted. 
	%For a ${\cal K}_t^j$ and ${\cal S}_t^j$ dependent bound on $Q_t^j$, refer to Lemma \ref{lemma.4} in Appendix \ref{appendix.aQ}.
\end{proof}

Similar to Corollary \ref{col.fixstepsize}, choosing $\eta_t^j =\sqrt{{\ln \check{K}}/(KT)}$ and $\mu_t^j = {\eta_t^j}/{2}, \forall t$, the regret of Algorithm \ref{algo2} without device cooperation is bounded by
	\begin{equation}\label{eq.afixreg}
	{\rm Reg}_T^j \leq 2\sqrt{TK\ln \check{K}}  \stackrel{(a)}{=} {\cal O} \big( \sqrt{TK^2\ln K}\big)
	\end{equation}
	where (a) follows from the Stirling's approximation $\ln \check{K} = K \ln K - K + {\cal O}(\ln K)$.
	If instead $\eta_t^j =\sqrt{\frac{\ln \check{K}}{2Kt}}$ and $\mu_t^j = \frac{\eta_t^j}{2}$, the regret of Algorithm \ref{algo2} without cooperation is bounded as
	\begin{equation}\label{eq.fixreg22}
	{\rm Reg}_T^j \leq 2\sqrt{2TK \ln \check{K}} =  {\cal O} \big( \sqrt{TK^2\ln K} \big).
	\end{equation}
	
The following corollary establishes the sublinear regret when SAVE-A is employed with cooperation. 
\begin{corollary}\label{col.afixstepsize}
	If $\eta_t^j = \sqrt{ (\ln \check{K}) / \big(K+ \sum_{\tau = 1}^{t-1}Q_\tau^j\big)}$, and $\mu_t^j =\eta_t^j/2, \forall t$, then
	\begin{equation}\label{eq.adareg2}
	{\rm Reg}_T^j \!\leq\! 2\sqrt{ \bigg(\delta + \sum_{t=1}^{T} Q_t^j \bigg) \!\ln \check{K}} \!=\! {\cal O}  \bigg( \sqrt{ \sum_{t=1}^{T} Q_t^j K\!\ln {K}} \bigg)\!\!\!
	\end{equation}
	where $\delta := \min_t \{ K- Q_t^j \}$.
\end{corollary}
\begin{proof}
The proof of \eqref{eq.adareg2} follows similar steps as that of Corollary \ref{col.adastepsize}, and thus it is omitted.
\end{proof}

%Corollary \ref{col.afixstepsize} states that fixed or diminishing $\eta_t^j$ and $\mu_t^j$ can guarantee an ${\cal O} \big( 2\sqrt{TK^2\ln K} \big)$ regret. The adaptive stepsizes in \eqref{eq.adareg2} specifies the value of cooperation in server list selection.
Compared with Corollaries \ref{col.fixstepsize} and \ref{col.adastepsize}, the bound in \eqref{eq.adareg2} has slightly worse dependence on $K$ due to the expansion of the search spaces. Similar to the stochastic case, we will rely on cooperation value $\lambda$ as the device-averaged ratio of the upper bound in \eqref{eq.adareg2} and that of \eqref{eq.afixreg}. Through $\lambda$, the ensuing corollary quantifies how side observations facilitate the security computing tasks.

\begin{corollary}\label{col.alambda}
	The cooperation value of SAVE-A satisfies
\begin{equation}\label{eq.scoop2}
\!\lambda \!\leq \!\frac{1}{J}\sum_{j=1}^J \!\sqrt{\frac{1}{T} \!+\! \frac{1}{KT}\!\sum_{t=1}^T\!  \Big( \big|{\cal K}_t^j\! \cup \!{\cal S}_t^j\big|\!-\!\big|{\cal S}_t^j\big| \!+\! \mathds{1} \big({\cal S}_t^j \!\neq \! \emptyset \big)\Big)}.\!\!\!\!
\end{equation}
\end{corollary}
\begin{proof}
	See Appendix \ref{appendix.aQ}.
\end{proof}

Similar to the stochastic jamming case, Corollary \ref{col.alambda} asserts that more side observations lead to lower security risk.

\section{Simulation Tests}
In this section, numerical tests are presented based on both synthetic and real data.

\begin{figure}[t]
	\vspace{-0.1cm}
	\centering
	\includegraphics[height=0.25\textwidth]{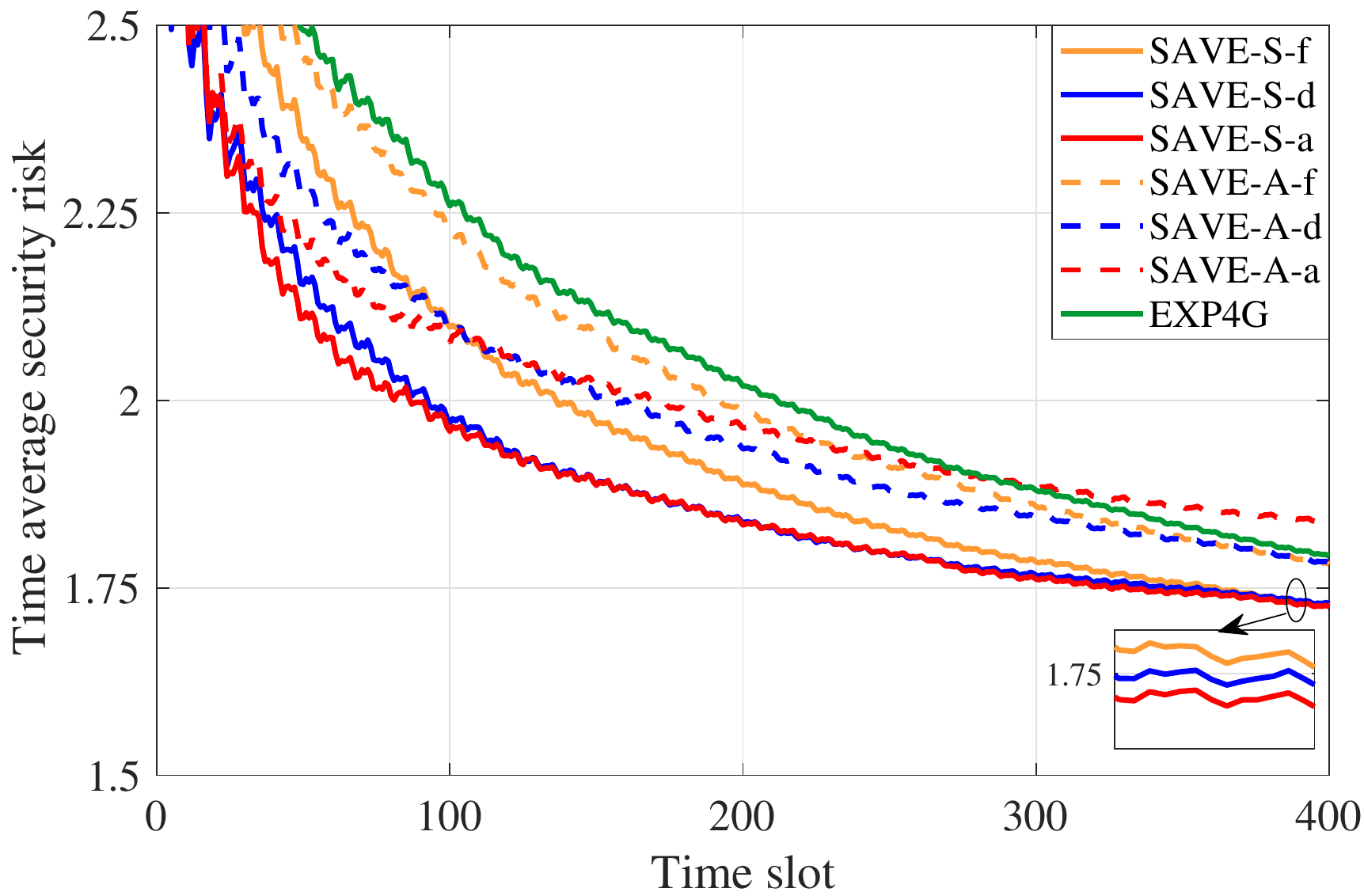}
	\vspace{-0.20cm}
	\caption{A comparison of SAVE-S and SAVE-A without cooperations under stochastic jamming attacks.}
	\label{fig.fStoSleepNoCoSingleDevice_magnify}
	\vspace{-0.1cm}
\end{figure}

\begin{figure*}[t]
	%\hspace{-0.3cm}
	\begin{tabular}{ccc}
		\hspace*{-2ex}
		\includegraphics[width=5.8cm]{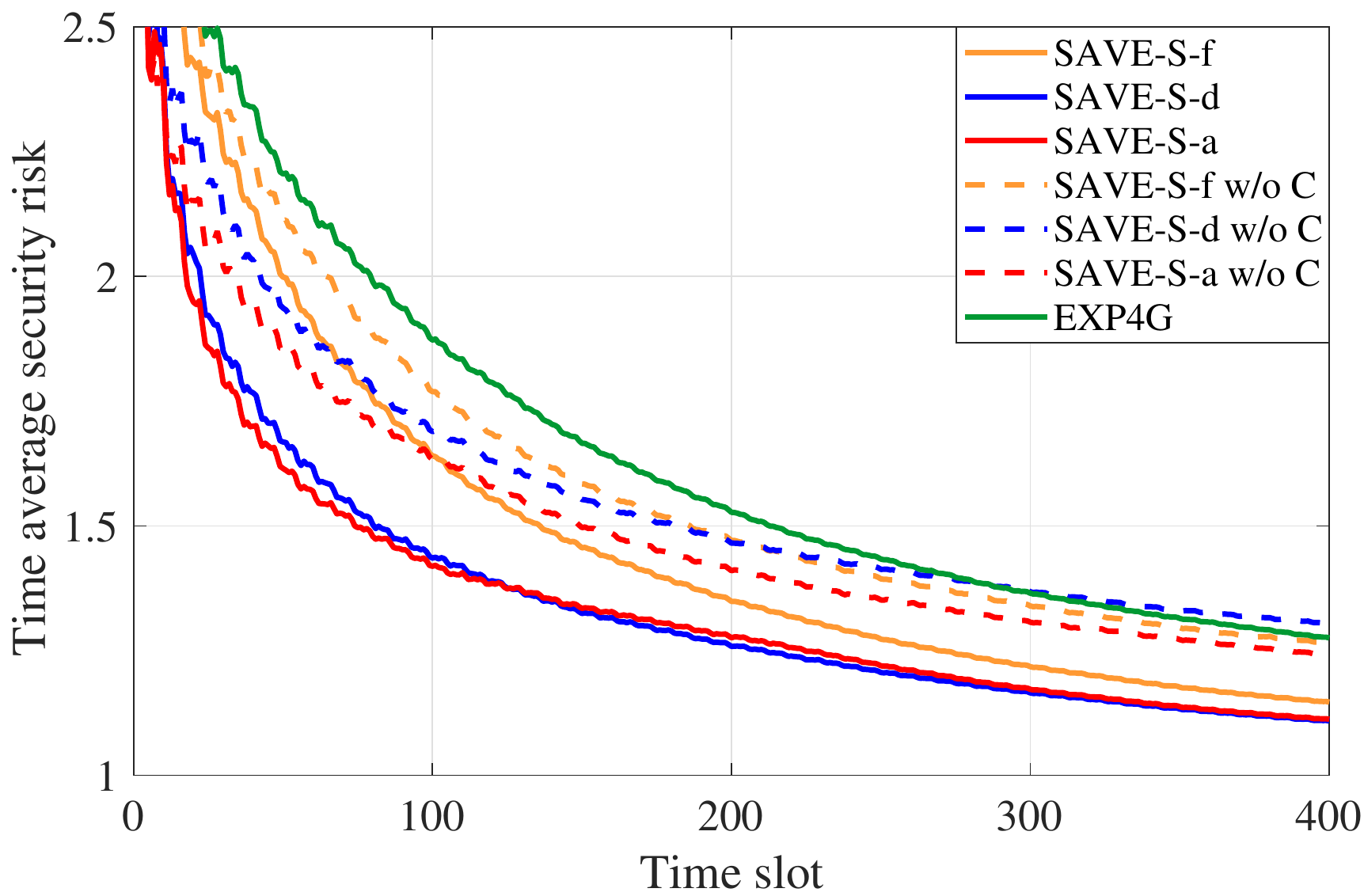}&
		\hspace*{-2ex}
		\includegraphics[width=5.8cm]{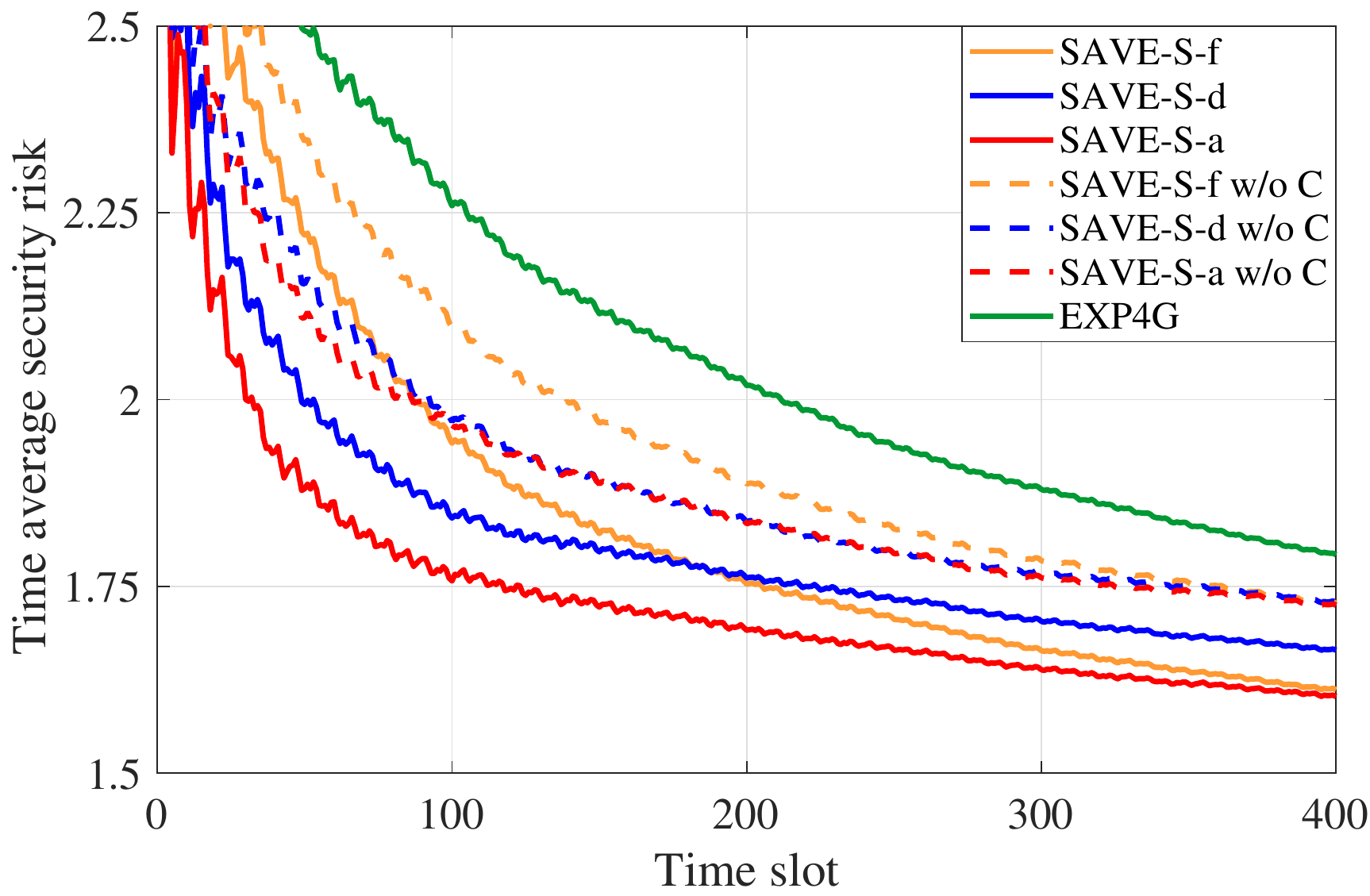}&
		\hspace*{-2ex}
		\includegraphics[width=5.8cm]{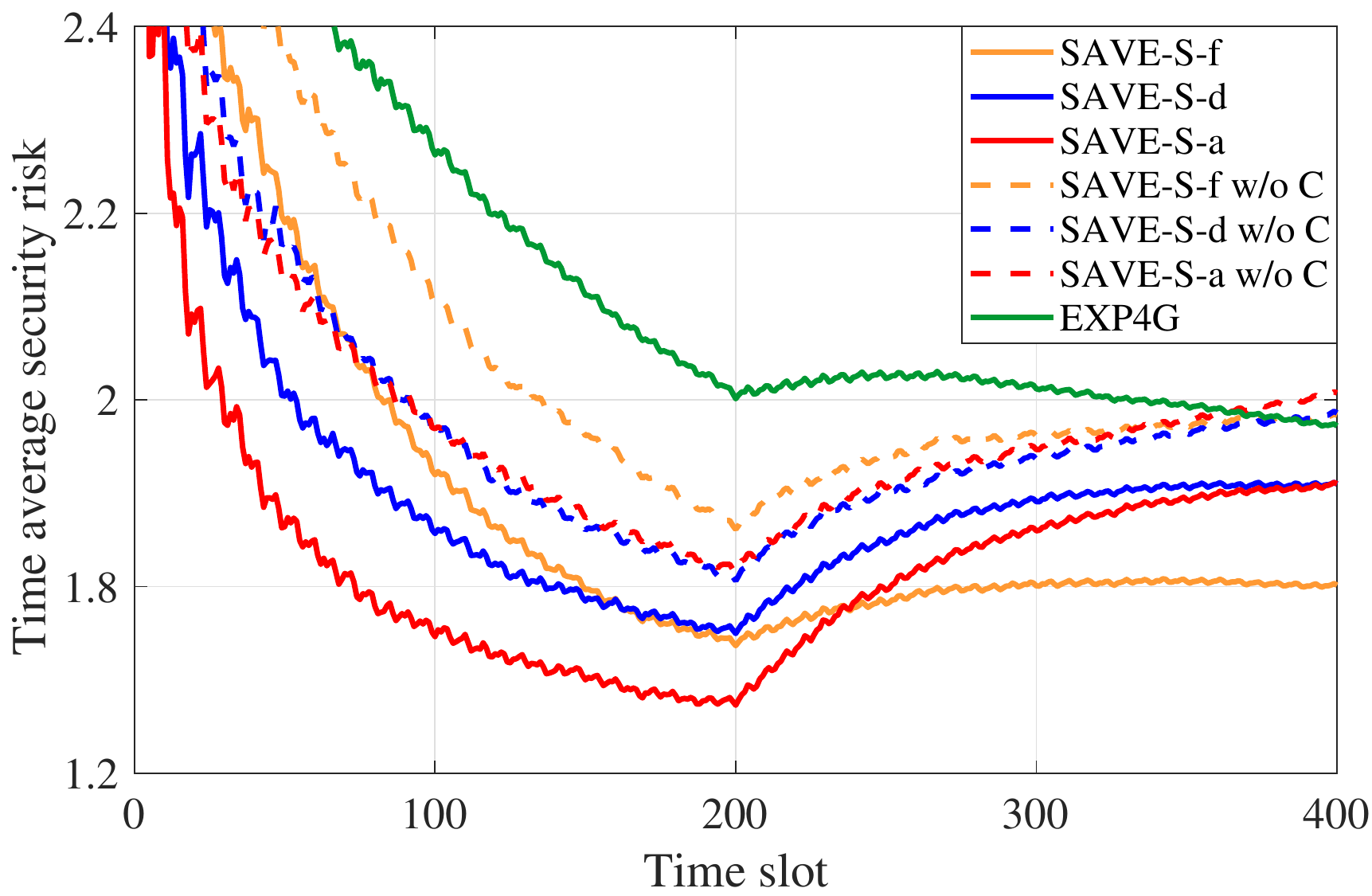}
		\\
		(a1)&(b1)&(c1)\\
		\includegraphics[width=5.8cm]{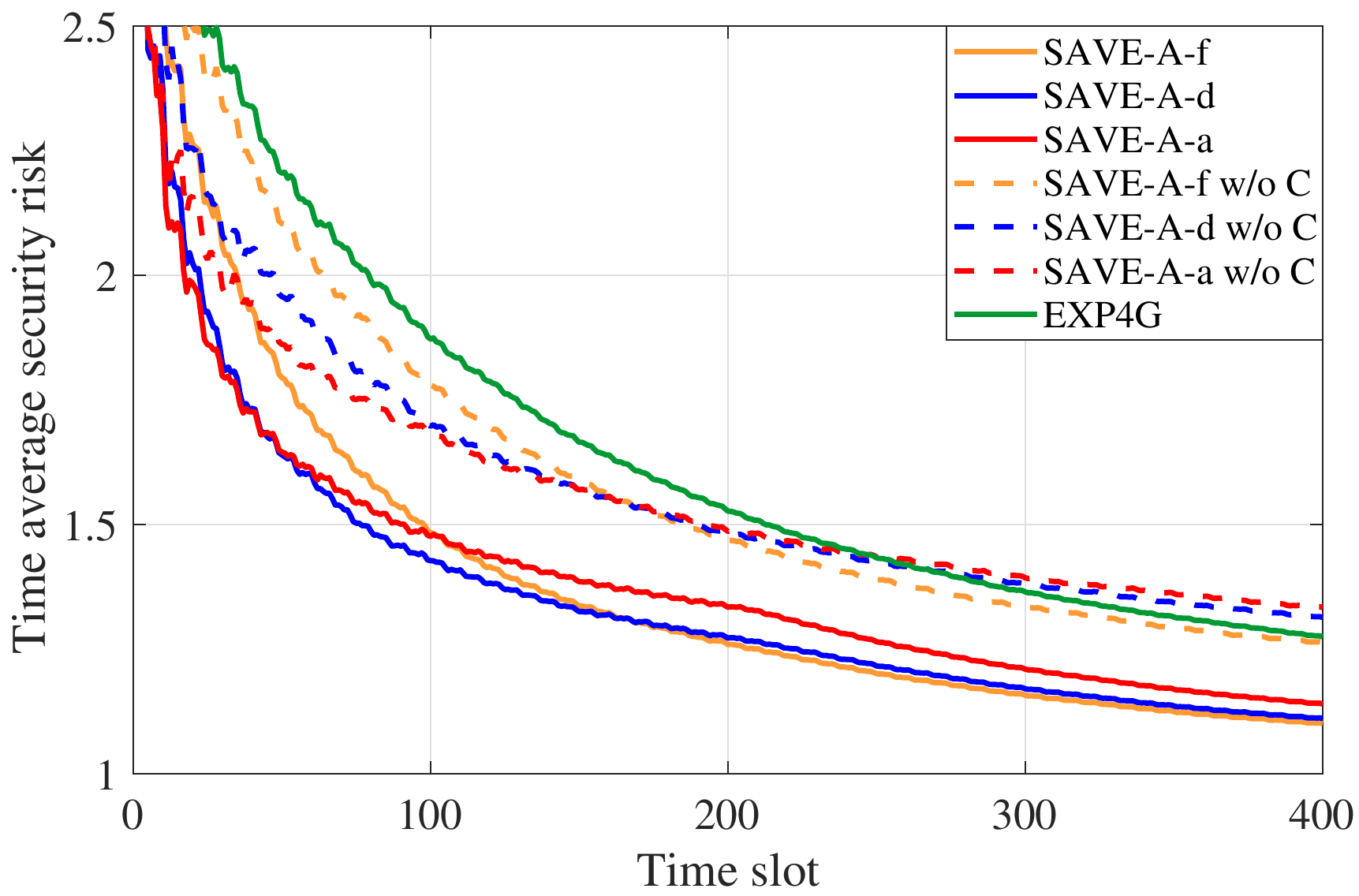}&
		\hspace*{-2ex}
		\includegraphics[width=5.8cm]{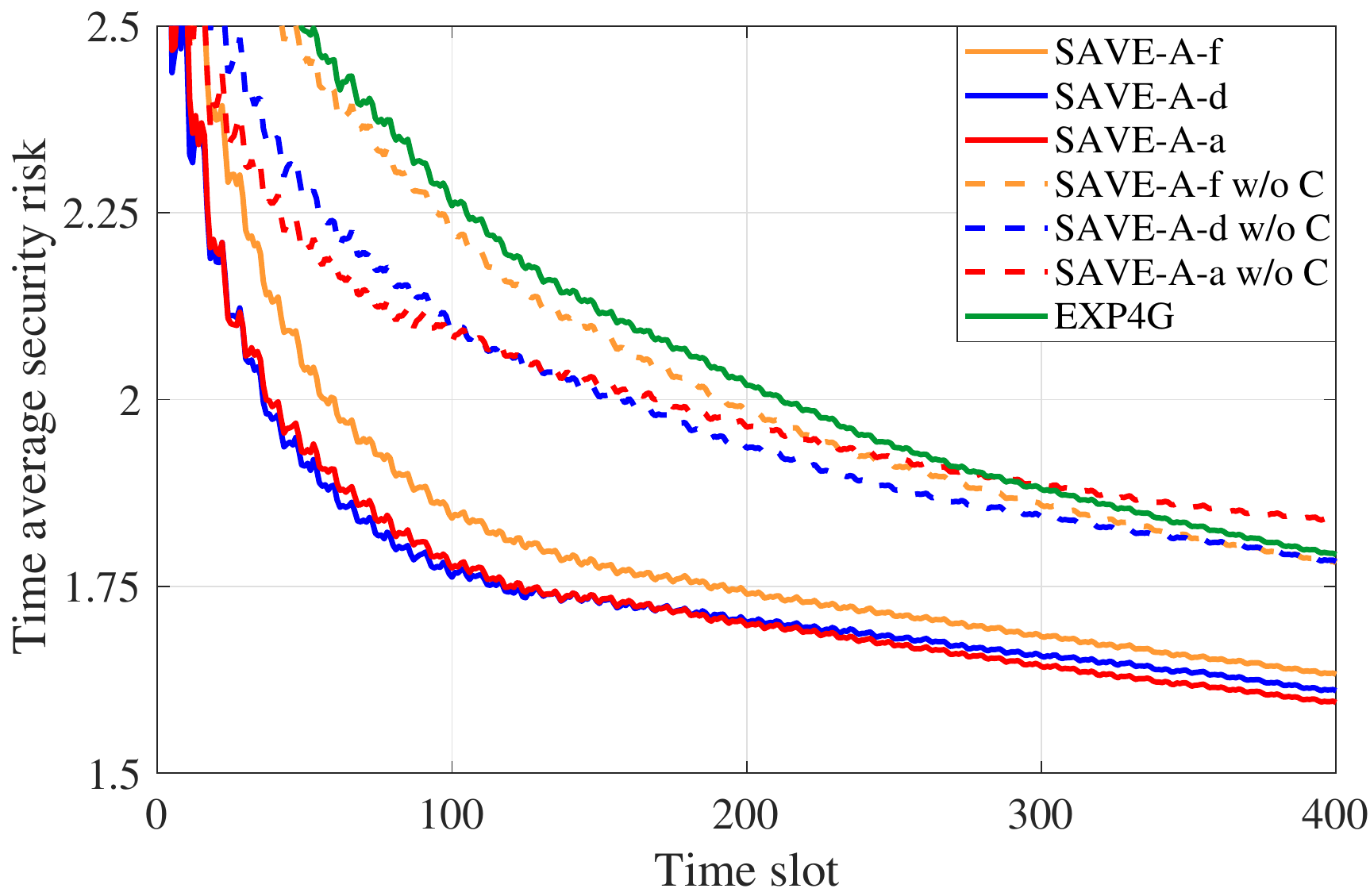}&
		\hspace*{-2ex}
		\includegraphics[width=5.8cm]{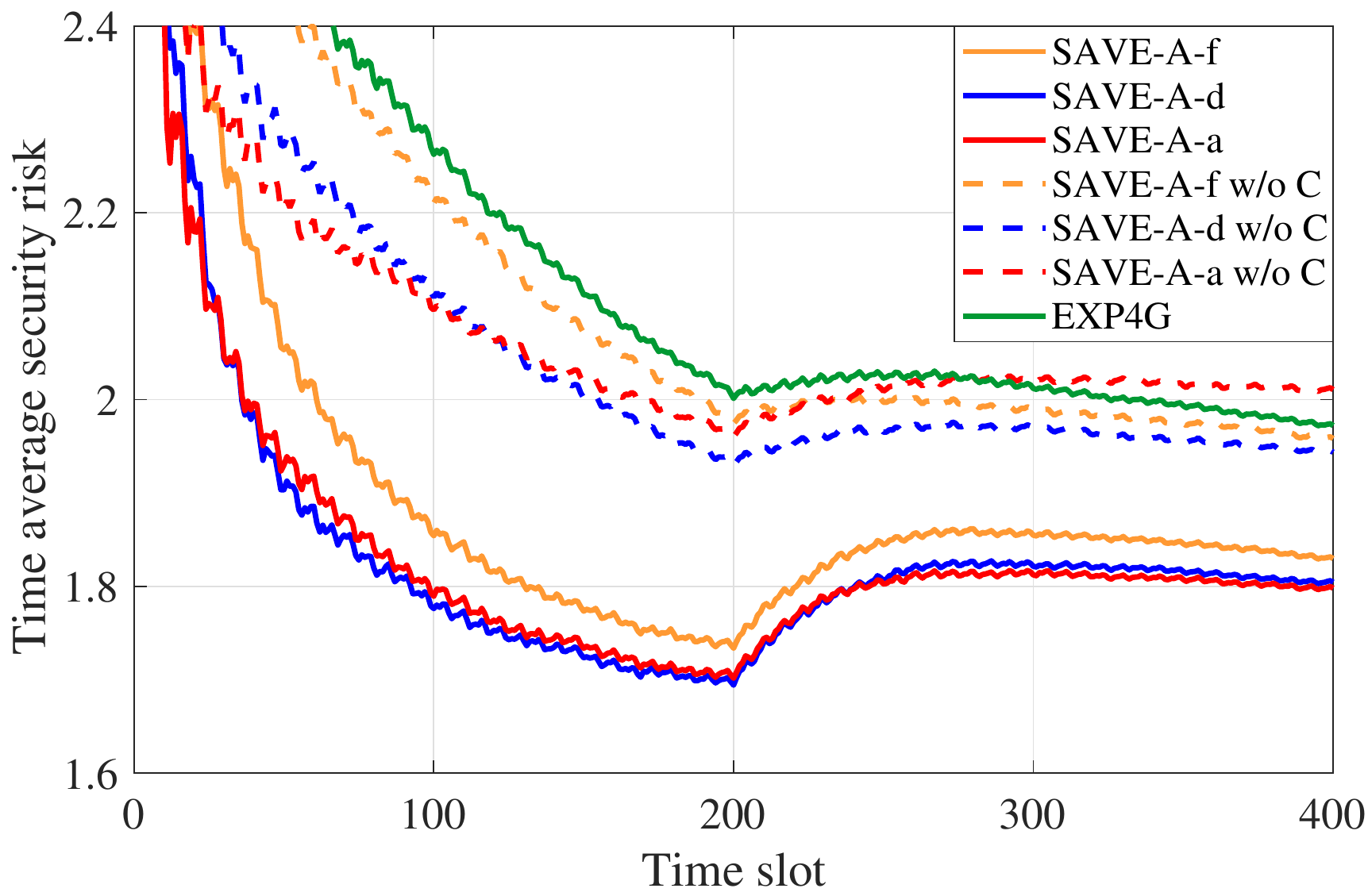}
		\\
		(a2)& (b2)&(c2)
	\end{tabular}
	\caption{Synthetic data tests: 
		(a1) SAVE-S without jamming attacks; (a2) SAVE-A without jamming attacks; (b1) SAVE-S with stochastic jamming attacks; (b2) SAVE-A with stochastic jamming attacks; (c1) SAVE-S with adversarial jamming attacks; (c2) SAVE-A with adversarial attacks.} \label{fig.cor}
\end{figure*}

\subsection{Synthetic data tests}
Our first test will rely on synthetic data. Consider $K=5$ edge servers, and $J = 1$ device with $\rho = 0.8$ over $T=400$ slots. 
The tasks' resource requirement $c_t$ is generated by $c_t = (0.6+0.5v_t)\cos{2t}$, 
%\begin{equation}
%c_t = (0.6+0.5v)\cos{2t}
%\end{equation}
where $v_t$ is uniformly distributed in $[0,1]$; and $s_t$ is given by $s_t= (0.25+0.3v_t)x_t$, 
%\begin{equation}
%s_t= (0.25+0.3v) x
%\end{equation}
where $v_t$ is again uniformly distributed in $[0,1]$; and $x_t$ is uniform random variable in $[0.8,1.2]$. For the corresponding security risks $\bm{\gamma}_{1,t}$ is generated by
\begin{equation}
	\gamma_{1,t} (k)= \frac{2k}{3} \big( |\sin{t}|+0.8 + |v_1|\big)
\end{equation}
with $v_1$ being a Gaussian random variable $v_1\sim {\cal N}(0,1.44)$, and $\bm{\gamma}_{2,t}$ is generated by
\begin{equation}
\gamma_{2,t} (k)= \frac{k}{2} \big( 0.5\sin{t}+0.75 + |v_2|\big)
\end{equation}
with $v_2\sim {\cal N}(0,0.64)$.

\begin{figure*}[t]
	%\hspace{-0.3cm}
	\begin{tabular}{ccc}
		\hspace*{-2ex}
		\includegraphics[width=5.9cm]{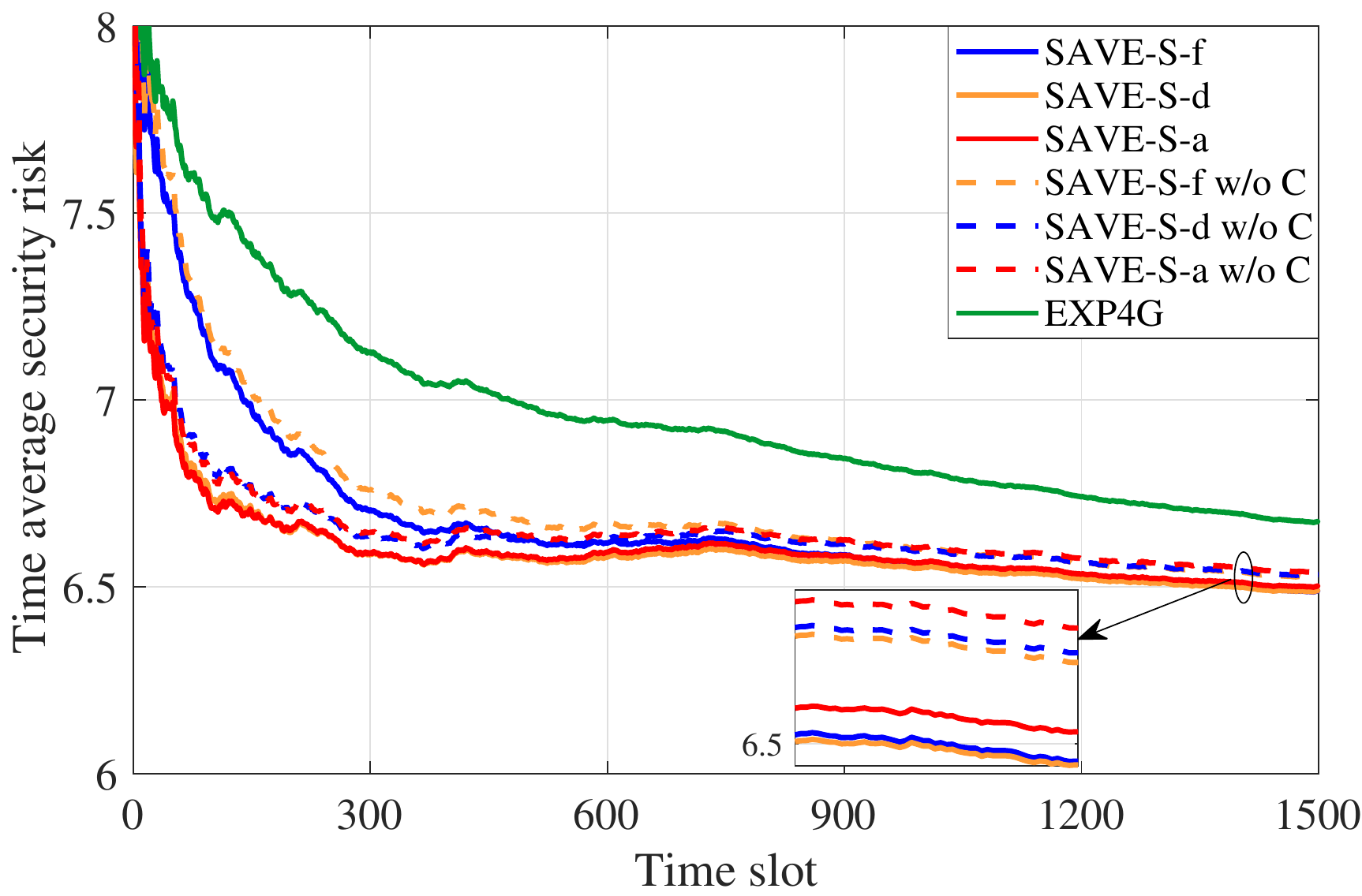}&
		\hspace*{-2ex}
		\includegraphics[width=5.9cm]{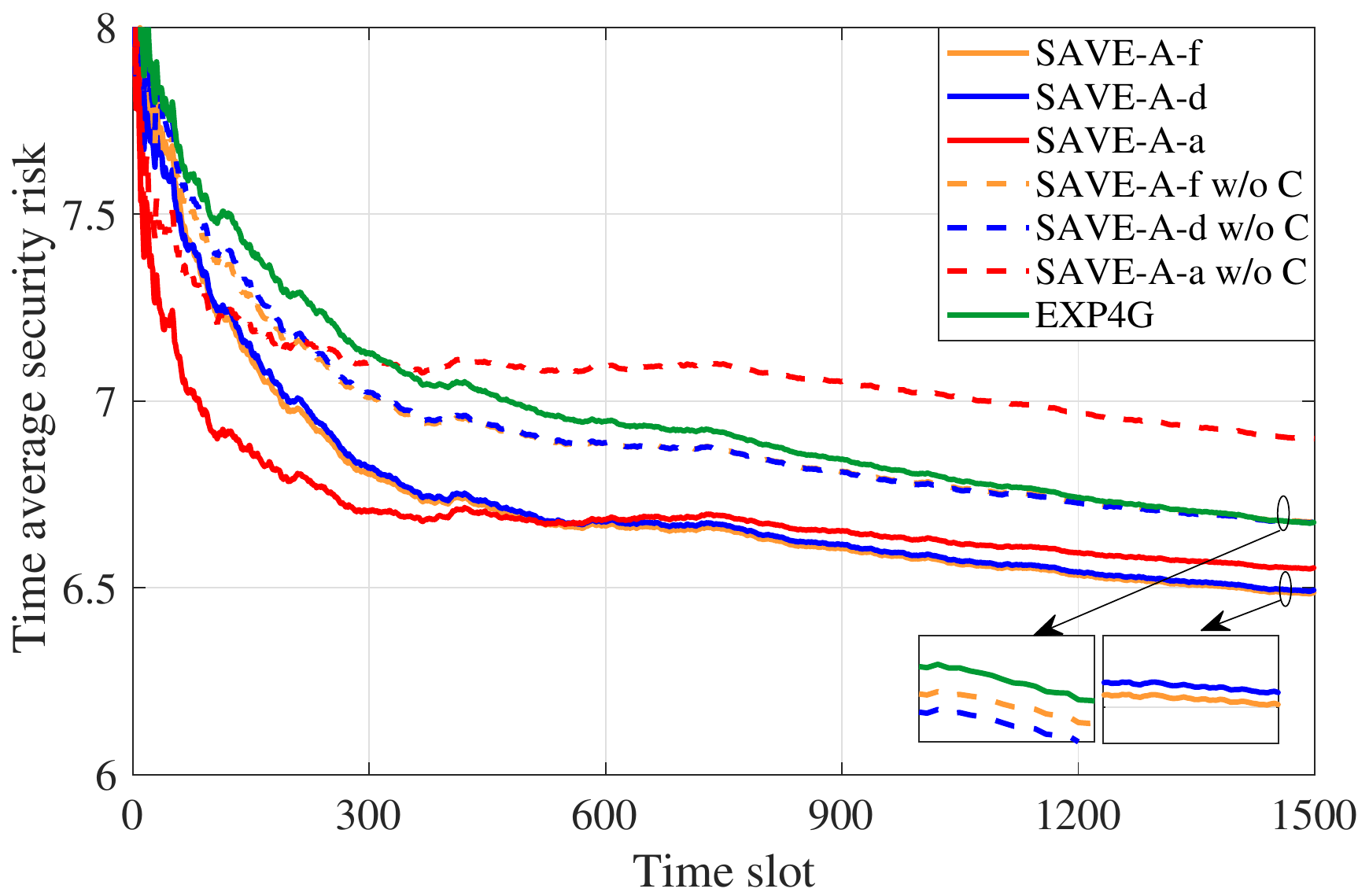}&
		\hspace*{-2ex}
		\includegraphics[width=5.9cm]{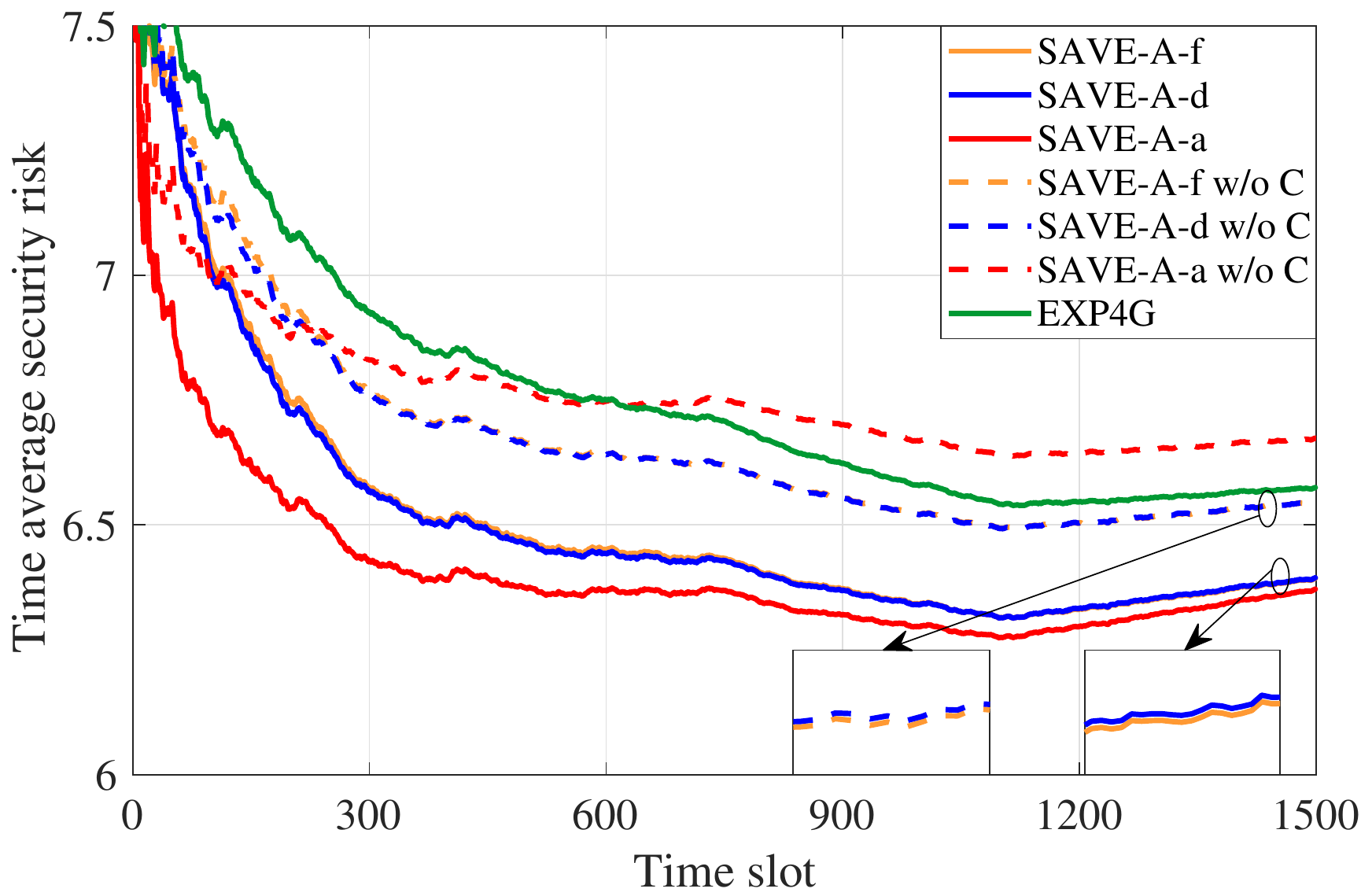}
		\\
		(a)& (b)&(c)
	\end{tabular}
	\caption{Real data tests: (a) SAVE-S for stochastic jamming attacks; (b) SAVE-A for stochastic attacks; (c) SAVE-A for adversarial jamming attacks.} \label{fig.real_data}
\end{figure*}

Before considering side observations, it is beneficial to demonstrate the effectiveness of SAVE-S and SAVE-A under a standard bandit setting without  cooperation among devices. For fairness, we consider the stochastic jammers, where SAVE-S and SAVE-A both enjoy theoretical guarantees. The jamming probability of edge servers is listed in the left part of Table \ref{tab.Sleep1}. The security risks of SAVE-S and SAVE-A with different stepsizes, and the EXP4G in \cite{li2018} slightly modified to suit our scenario, are plotted in Fig. \ref{fig.fStoSleepNoCoSingleDevice_magnify}. 
The proposed SAVE-S and SAVE-A with different stepsizes all outperform EXP4G except for SAVE-A with adaptive stepsizes. It can be seen that, SAVE-S outperforms SAVE-A, which can be explained because the number of arms in SAVE-S is smaller and thus fewer exploitation are needed. 

To showcase the improvement attained from side observations, instead of receiving security risks directly from other devices, an alternative method is adopted where the side observations are obtained probabilistically. Specifically, for the first $200$ slots, the probability of revealing each edge server's risk is listed in the white part of Table \ref{tab.Co1}; and for the rest of the slots, it is listed in the blue part of Table \ref{tab.Co1}. 

\begin{table}[t]\addtolength{\tabcolsep}{0pt}
	%\vspace{-0.2cm}
	\centering \caption{Side Observation (SO) Probability}\label{tab.Co1} 
	\vspace{-0.1cm}
	\begin{tabular}{ c *{5}{|c}||c *{4}{|c}}
		\hline
		Server   & S1 & 	S2  & S3 & S4& S5& S1 & 	S2  & S3 & S4 & S5 \\ \hline
		Is SO & 1 &  1  & 0 & 0 & 1 & \cellcolor{blue!15} 0.3 & 	\cellcolor{blue!15}1  & \cellcolor{blue!15}0.6 & \cellcolor{blue!15}0.5 & \cellcolor{blue!15}0\\
		\hline
		Not SO & 0 & 0  & 1 & 1 & 0 & \cellcolor{blue!15}0.7 & \cellcolor{blue!15}0  & \cellcolor{blue!15}0.4 & \cellcolor{blue!15}0.5 & \cellcolor{blue!15}1\\
		\hline
	\end{tabular} 
	\vspace{-0.2cm}
\end{table}

%\begin{table}[t]\addtolength{\tabcolsep}{0pt}
%	\vspace{-0.2cm}
%	\centering \caption{Side Observation Probability}\label{tab.Co2} \vspace{0.1cm}
%	\begin{tabular}{ c *{5}{|c}c *{4}{|c}}
%		\hline
%		Server   & S1 & 	S2  & S3 & S4 & S5 \\ \hline
%		Is SO & 0.3 & 	1  & 0.6 & 0.5 & 0 \\
%		\hline
%		Not SO& 0.7 & 0  & 0.4 & 0.5 & 1 \\
%		\hline
%	\end{tabular} 
%		\vspace{-0.2cm}
%\end{table}

\begin{table}[t]\addtolength{\tabcolsep}{0pt}
	%\vspace{-0.2cm}
	\centering \caption{Server On/Off Probability}\label{tab.Sleep1} 
	\vspace{-0.1cm}
	\begin{tabular}{ c *{5}{|c}||c *{4}{|c}}
		\hline
		Server   & S1 & 	S2  & S3 & S4& S5& S1 & 	S2  & S3 & S4
		& S5 \\ \hline
		On & 0.7 & 	0.8  & 0.9 & 1 & 0.6 & \cellcolor{blue!15}0.3 & 	\cellcolor{blue!15}1  & \cellcolor{blue!15}0.6 & \cellcolor{blue!15}0.5 & \cellcolor{blue!15}0.8 \\
		\hline
		Off & 0.3 & 0.2  & 0.1 & 0 & 0.4 &\cellcolor{blue!15}0.7 & \cellcolor{blue!15}0  & \cellcolor{blue!15}0.4 & \cellcolor{blue!15}0.5 & \cellcolor{blue!15}0.2\\
		\hline
	\end{tabular} 
	\vspace{-0.2cm}
\end{table}

\begin{table}[t]\addtolength{\tabcolsep}{0pt}
	%\vspace{-0.2cm}
	\centering 
	\caption{Server On/Off Probability}\label{tab.CorRealData} \vspace{-0.1cm}
	\begin{tabular}{ c *{6}{|c}}
		\hline
		Link   & 1 to 2 & 1 to 3  & 2 to 1 & 2 to 3
		& 3 to 1 & 3 to 2 \\ \hline
		Cooperation & 0.1 & 0.4  & 0 & 0.5 & 0.6&0.3 \\
		\hline
		No Cooperation & 0.9 & 0.6  & 1 & 0.5 & 0.4& 0.7\\
		\hline
	\end{tabular} 
	\vspace{-0.2cm}
\end{table}

\textbf{No jamming:} In Figs. \ref{fig.cor} (a1) and (a2), the SAVE-S and SAVE-A are compared with their corresponding non-cooperative variants, respectively. Clearly, the cooperation improves the regret of SAVE-S by a percentage of $54.49\%$, $53.08\%$ and $47.47\%$ for fixed, diminishing, and adaptive stepsizes, respectively. Regarding SAVE-A, the improvement thanks to cooperation is $50.22\%$, $52.17\%$ and $50.03\%$  along with fixed, diminishing, and adaptive stepsizes. As confirmed by simulations, the cooperation significantly improves the regret performance of SAVE-S and SAVE-A; e.g., the cooperation values are $\lambda = 0.5074$ and $\lambda = 0.4985$, respectively.

\textbf{Stochastic jamming attacks:} Suppose that the servers are under attack by stochastic jammers, where the on-off probability of edge servers is listed in Table \ref{tab.Sleep1}. The simulations shown in Figs. \ref{fig.cor} (b1) and (b2) illustrate that the cooperation improves the regret of SAVE-S by a percentage of $28.53\%$, $16.34\%$ and $29.07\%$ for fixed, diminishing, and adaptive stepsizes, respectively. In the present test, the cooperation value is $\lambda = 0.4959$. Regarding SAVE-A, the improvement provided by cooperation is $37.33\%$, $34.31\%$, and $45.84\%$ when fixed, diminishing, and adaptive stepsizes are adopted, where the cooperation value is $\lambda = 0.5196$.

\textbf{Adversarial jamming attacks:}
With adversarial jammers, the difference in data generation is that in the first $200$ slots the probability of server being jammed follows the left part of Table \ref{tab.Sleep1}, while the rest of the time the probability follows the right part of Table \ref{tab.Sleep1}. Fig. \ref{fig.cor} (c1) depicts the performance of SAVE-S, which is not guaranteed to obtain sublinear regret. On the other hand, Fig. \ref{fig.cor} (c2) compares the performance of SAVE-A with different stepsizes. The cooperation improves the regret of SAVE-A by a percentage of $21.82\%$, $30.24\%$ and $37.75\%$ for fixed, diminishing, and adaptive stepsizes respectively, along with $\lambda = 0.5412$.

\subsection{Real data tests}

The performance of SAVE-S and SAVE-A is further tested on a real world dataset \cite{noor2016}, which contains the customers' feedback on cloud service from public websites such as Cloud Hosting Reviews, where more than 10,000 feedback information form nearly 7,000 consumers over 113 cloud services are collected. The consumers' feedback is the service trust (using risk $\bm{\gamma}_{1,t}$ and $\bm{\gamma}_{2,t}$ for negative trust). In this test, we consider $K=3$ edge servers and $J = 3$ IoT devices. The information sharing probability values are listed in Table \ref{tab.CorRealData}.

The SAVE-S for combating stochastic jammers is shown in Fig. \ref{fig.real_data} (a). SAVE-S with different stepsizes outperform EXP4G with diminishing stepsize exposing the smallest average security risk. In this case, cooperation slightly improves the regret performance at a percentage of $20.54\%$, $23.52\%$ and $19.38\%$ for fixed, diminishing, and adaptive stepsizes. The cooperation value in this case is $\lambda = 0.7123$. The SAVE-A under stochastic jamming is shown in Fig. \ref{fig.real_data} (b). Even without cooperation, SAVE-A with fixed stepsize and diminishing stepsize outperforms EXP4G. Meanwhile, the cooperation improves the regret $58.87\%$, $53.70\%$ and $56.96\%$ for fixed, diminishing, and adaptive stepsizes respectively, together with a cooperation value $\lambda = 0.6301$. 

\begin{figure}[t]
	\vspace{-0.1cm}
	\centering
	\includegraphics[height=0.28\textwidth]{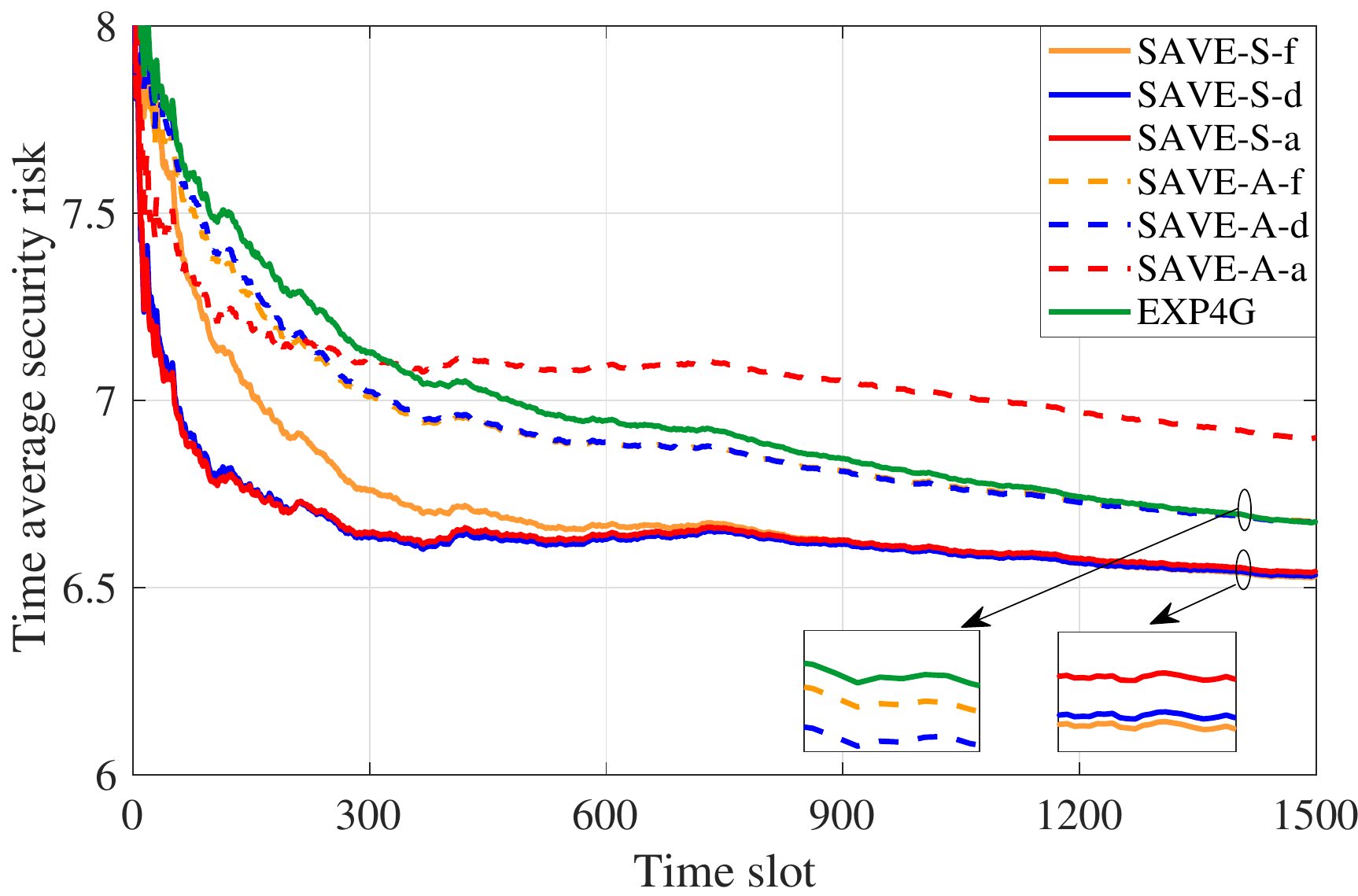}
	\vspace{-0.2cm}
	\caption{A comparison of SAVE-S and SAVE-A using real data.}
	\label{fig.fSA_comp_StoSleep_realdata_magnify}
	\vspace{-0.2cm}
\end{figure}

Regarding adversarial jammers, Fig. \ref{fig.real_data} (c) shows how cooperation improves the regret of SAVE-A by a percentage of $50.18\%$, $49.52\%$ and $63.83\%$ for fixed, diminishing, and adaptive stepsizes, respectively; while the cooperation value is $\lambda = 0.6265$.
 We further compare the time averaged security risk in SAVE-S with SAVE-A in Fig. \ref{fig.fSA_comp_StoSleep_realdata_magnify} under stochastic jamming attacks without cooperations. It is seen that SAVE-S with different stepsizes all outperform SAVE-A.

\section{Conclusions}
Online security-aware edge computing under jamming attacks was studied in this paper. Different from the common ways such as expanding spectrum or increasing transmission power, we developed schemes suitable for low-power IoT devices. Specifically, we developed our SAVE-S and SAVE-A algorithms to offload tasks to the most reliable server under stochastic and adversarial jamming attacks, respectively. Sublinear regret for both schemes was analytically established. 
Performance of SAVE-S and SAVE-A was further enhanced via cooperation among devices. 
Analysis confirmed the value of cooperation via the marked improvement on the regret bound. Numerical tests on both synthetic and real datasets demonstrated the effectiveness of the proposed schemes.

\appendix

\subsection{Proof of Theorem \ref{theo.1}}\label{appendix.stheo}
The proof starts with a simple case, where for a single device $j$, we have ${\cal K}_t^j = \tilde{\cal K}, ~\forall t$. 
%In this case competitor in \eqref{eq.sreg} reduces to the best single server in hindsight.
\begin{lemma}\label{lemma.2}
	If ${\cal K}_t^j = \tilde{\cal K}, ~\forall t$, then SAVE-S guarantees that
	\begin{align}
	 \sum_{t=1}^{T}\sum_{k=1}^{K}  p_t^j (k) & {r}_{t}^{j}(k)  -  \sum_{t=1}^{T}  r_t^{j} (k^*) \!\leq\! \sum_{t=1}^{T} \bigg( \mu_t^j\!+\!\frac{\eta_t^j}{2} \bigg) Q_t^j\! +\! \frac{\ln K}{\eta_{T+1}^j}
	\end{align}
where $k^* $ denotes the best fixed server among $\tilde{\cal K}$ in hindsight.
\end{lemma}

\begin{proof}

Upon defining auxiliary variables $W^j_t \!:=\! \sum_{k=1}^{K} \exp \big[ -\eta_{t}^j \hat{R}_{t-1}^j (k) \big]$, and $\tilde{W}^j_t \!:=\! \sum_{k=1}^{K} \exp \big[ \!-\eta_{t-1}^j \hat{R}_{t-1}^j (k)\big]$,
%\begin{equation}
%\!	W^j_t \!=\! \sum_{k=1}^{K} \exp \big[ -\eta_{t}^j \hat{R}_{t-1}^j (k) \big]; ~ \tilde{W}^j_t \!=\! \sum_{k=1}^{K} \exp \big[ \!-\eta_{t-1}^j \hat{R}_{t-1}^j (k)\big].\!
%\end{equation}
we have 
\begin{align}\label{eq.39}
	&\frac{1}{\eta_t^j} \ln \bigg( \frac{\tilde{W}^j_{t+1}}{W^j_t}\bigg) = \frac{1}{\eta_t^j}\ln \bigg( \frac{\sum_{k=1}^{K} \exp \big[ -\eta_{t}^j \hat{R}_{t}^j (k)\big]}{W^j_t}\bigg) \nonumber \\
	=&\frac{1}{\eta_t^j}\ln \bigg(\sum_{k=1}^{K} \frac{ w_t^j(k)\exp \big[ -\eta_{t}^j \hat{r}_{t}^j (k)\big]}{W^j_t}\bigg) \nonumber \\
	=& \frac{1}{\eta_t^j}\ln \bigg(\sum_{k=1}^{K}  p_t^j(k)\exp \big[ -\eta_{t}^j \hat{r}_{t}^j (k)\big]\bigg)\nonumber\\
		\stackrel{(a)}{\leq} &\frac{1}{\eta_t^j}\ln \bigg(\sum_{k=1}^{K}  p_t^j(k)  \Big(1-\eta_t^j \hat{r}_{t}^j (k) +\frac{1}{2} \big(\eta_t^j \hat{r}_{t}^j (k) \big)^2 \Big)   \bigg)
\end{align}
where (a) is due to $e^{-x} \leq 1- x +\frac{x^2}{2},~\forall x \geq 0$. The bound in \eqref{eq.39} can be further bounded as
\begin{align}
&\frac{1}{\eta_t^j}\ln \bigg(\sum_{k=1}^{K}  p_t^j(k)  \Big(1-\eta_t^j \hat{r}_{t}^j (k) +\frac{1}{2} \big(\eta_t^j \hat{r}_{t}^j (k) \big)^2 \Big)   \bigg) \nonumber \\
	 =& \frac{1}{\eta_t^j}\ln \bigg(1  -\eta_t^j \sum_{k=1}^{K} \hat{r}_{t}^j (k)p_t^j(k) +\frac{\big(\eta_t^j \big)^2}{2} \sum_{k=1}^{K} p_t^j(k) \big( \hat{r}_{t}^j (k) \big)^2   \bigg) \nonumber \\
	\stackrel{(b)}{\leq} &- \sum_{k=1}^{K} p_t^j(k) \hat{r}_{t}^j (k) +\frac{\eta_t^j }{2} \sum_{k=1}^{K} p_t^j(k) \big(\hat{r}_{t}^j (k) \big)^2
\end{align}
where (b) follows from $\ln (1-x)\leq -x, ~\forall x \geq 0$. Therefore, we have
\begin{equation}\label{eq.step1}
	\frac{1}{\eta_t^j} \ln \bigg( \frac{\tilde{W}^j_{t+1}}{W^j_t}\bigg)\leq - \sum_{k=1}^{K} p_t^j(k) \hat{r}_{t}^j (k) +\frac{\eta_t^j }{2} \sum_{k=1}^{K} p_t^j(k) \big(\hat{r}_{t}^j (k) \big)^2.
\end{equation}

Rearranging \eqref{eq.step1}, we arrive at
\begin{align}\label{eq.step2}
\!\!&	\sum_{k=1}^{K}  p_t^j(k) \hat{r}_{t}^j (k)\leq \frac{\eta_t^j }{2} \sum_{k=1}^{K} p_t^j(k) \big(\hat{r}_{t}^j (k) \big)^2 \!+\! \frac{1}{\eta_t^j} \ln \frac{W^j_t}{\tilde{W}^j_{t+1}}\\ 
\!\!=& \frac{\eta_t^j }{2}\! \sum_{k=1}^{K} p_t^j(k) \big(\hat{r}_{t}^j (k) \big)^2\!  
	\!+\!\bigg(\!\frac{\ln W^j_t}{\eta_t^j}\!-\!\frac{\ln W^j_{t+1}}{\eta_{t+1}^j}\!\bigg)\!\!+\!\!\bigg(\!\frac{\ln W^j_{t+1}}{\eta_{t+1}^j}\!-\!\frac{\ln \tilde{W}^j_{t+1}}{\eta_{t}^j}\!\bigg). \nonumber 
\end{align}
To bound $\frac{\ln W^j_{t+1}}{\eta_{t+1}^j} - \frac{\ln \tilde{W}^j_{t+1}}{\eta_{t}^j}  $, notice that
\begin{align}
\!\!W^j_{t+1}& \!=\! \sum_{k=1}^{K}\!\exp\! \big[\!-\eta_{t+1}^{j} \hat{r}_t^j (k) \big]\! \nonumber \\
	&\stackrel{(c)}{\leq}\! K\! \bigg( \sum_{k=1}^{K} \frac{1}{K} \exp[-\eta_{t}^{j} \hat{r}_t^j (k)]\!\bigg)^{{\eta_{t+1}^j}/{\eta_{t}^{j}}} \nonumber \\ 
	&= K^{({\eta_t^j - \eta_{t+1}^j})/{\eta_t^j}} \bigg(\sum_{k=1}^{K} \exp \big[ -\eta_{t}^{j} \hat{r}_t^j (k) \big] \bigg)^{{\eta_{t+1}^j}/{\eta_{t}^{j}}} \nonumber \\
	&= K^{({\eta_t^j - \eta_{t+1}^j})/{\eta_t^j}} (\tilde{W}^j_{t+1})^{{\eta_{t+1}^{j}}/{\eta_{t}^{j}}}
\end{align}
where (c) stems from $\eta_{t+1}^j \leq \eta_{t}^j$, and the concavity of $(\cdot)^{\eta_{t+1}^j/\eta_{t}^j}$. Taking logarithms on both sides, and rearranging terms leads to 
\begin{align}\label{eq.step3}
\frac{\ln W^j_{t+1}}{ \eta_{t+1}^j} - \frac{\ln \tilde{W}^j_{t+1}}{ \eta_{t}^j} \leq \bigg( \frac{1}{\eta_{t+1}^j} - \frac{1}{\eta_{t}^j}\bigg) \ln K.
\end{align}

Plugging \eqref{eq.step3} into \eqref{eq.step2} and summing up over $t=1,2,\ldots,T$, we arrive at
\begin{align}\label{eq.step4}
\!\!\!\!&\sum_{t=1}^{T} \sum_{k=1}^{K}  p_t^j(k) \hat{r}_{t}^j (k)\leq \sum_{t=1}^{T} \frac{\eta_t^j }{2} \sum_{k=1}^{K} p_t^j(k) \big(\hat{r}_{t}^j (k) \big)^2 \nonumber \\
\!\!\!\!\!	 &\qquad\qquad\qquad+ \bigg(\frac{1}{\eta_{T+1}^j}-\frac{1}{\eta_1^j} \bigg) \ln K + \frac{\ln W^j_1}{\eta_1^j} -  \frac{\ln W^j_{T+1}}{\eta_{T+1}^j} \nonumber \\
\!\!\!\!\!	\! & \stackrel{(d)}{=}\sum_{t=1}^{T} \!\frac{\eta_t^j }{2}\! \sum_{k=1}^{K}\! p_t^j(k)\!  \big(\hat{r}_{t}^j (k)\big)^2 \!+\! \bigg(\!\frac{1}{\eta_{T+1}^j}\!-\!\frac{1}{\eta_1^j} \!\bigg) \ln K \!-\! \frac{\ln W^j_{T+1}}{\eta_{T+1}^j}\!\!\!
\end{align}
where (d) follows from $W^j_1 = 1$. Since the estimators in \eqref{eq.est} can be written compactly using \eqref{eq.est3}, with $\tilde{p}_t^{j}(k):= \sum_{(m, k) \in {\cal G}_t^j} p_t^j(m)$, we have
\begin{align}\label{eq.1st}
	&\mathbb{E}\bigg[ \sum_{k=1}^{K}  p_t^j(k) \hat{r}_{t}^j (k) \bigg] = \sum_{k=1}^{K} p_t^j(k) r_t^j (k) \bigg(1  -\frac{\mu_t^j}{\mu_t^j+\tilde{p}_t^{j}(k)} \bigg) \nonumber \\
	&= \sum_{k=1}^{K} p_t^j(k) r_t^j (k) - \sum_{k=1}^{K} p_t^j(k) r_t^j (k) \frac{\mu_t^j}{\mu_t^j+\tilde{p}_t^{j}(k)}  \nonumber \\
	& \stackrel{(e)}{\geq} \sum_{k=1}^{K} p_t^j(k) r_t^j (k) - \mu_t^j Q_t^j 
\end{align}
where $\mathbb{E}$ is w.r.t. the probability that $r_t^j (k)$ is observed; and (e) follows from the assumption $r_t^j (k) \leq 1$ and the definition of $Q_t^j$. The mean square in \eqref{eq.step4} can be bounded as
\begin{align}\label{eq.2nd}
	&\mathbb{E}\bigg[ \sum_{k=1}^{K}  p_t^j(k) \big( \hat{r}_{t}^j (k) \big)^2 \bigg] \leq \sum_{k=1}^{K}  p_t^j(k) \tilde{p}_t^{j}(k) \frac{\big(r_t^j (k)\big)^2}{ \big( \mu_t^j+\tilde{p}_t^{j}(k)\big)^2} \nonumber \\
\leq &\sum_{k=1}^{K}  p_t^j(k)\frac{\big(r_t^j (k)\big)^2}{ \mu_t^j+\tilde{p}_t^{j}(k)} \leq \sum_{k=1}^{K}  p_t^j(k)\frac{1}{ \mu_t^j+\tilde{p}_t^{j}(k)} \leq Q_t^j.
\end{align}

For the third term in \eqref{eq.step4}, we have 
\begin{align}\label{eq.step5}
\!\!	& \mathbb{E} \bigg[ -\frac{\ln W^j_{T+1}}{\eta_{T+1}^j} \bigg] = \mathbb{E} \bigg[ -\frac{\ln \sum_{k=1}^{K}w_{T+1}^j (k)}{\eta_{T+1}^j} \bigg] \nonumber \\
 \!\!\!\stackrel{(f)}{\leq} &\mathbb{E} \bigg[ \!-\! \frac{\ln \sum_{k=1}^{K} p^j(k)w_{T+1}^k}{\eta_{T+1}^j} \bigg] \stackrel{(g)}{\leq} \mathbb{E} \bigg[ \!-\! \sum_{k=1}^{K} p^j(k)\frac{\ln  w_{T+1}^j(k)}{\eta_{T+1}^j} \bigg] \nonumber \\
 \!\!\!\!= &\mathbb{E} \bigg[\!\!-\!\sum_{k=1}^{K} p^j(k)\frac{ - \eta_{T+1}^j \hat{R}_{T+1}^j (k)}{\eta_{T+1}^j} \bigg] \!\!\stackrel{(h)}{\leq} \!\sum_{k=1}^{K}p^j(k)\!\sum_{t=1}^{T}\! r_t^j (k)\!\!
\end{align}
where (f) follows since $\ln (\cdot)$ is monotonically increasing and $\{p^j(k)\}$ is a fixed distribution; (g) is due to Jensen's inequality; and (h) follows since $\hat{r}_t^j(k)$ as well as $\hat{R}_t^j(k)$ are underestimators. Taking expectation on \eqref{eq.step4}, and combining \eqref{eq.1st}-\eqref{eq.step5}, we have
\begin{align}\label{eq.lemma2reg}
	\sum_{t=1}^T\sum_{k=1}^{K} & p_t^j(k) {r}_{t}^{j}(k) -\sum_{t=1}^{T} \sum_{k=1}^{K}p^j(k) r_t^j (k) \nonumber \\
	&\leq \sum_{t=1}^{T} \bigg( \mu_t^j+ \frac{\eta_t^j}{2} \bigg) Q_t^j + \bigg(\frac{1}{\eta_{T+1}^j}-\frac{1}{\eta_1^j} \bigg) \ln K
\end{align}
which completes the proof.
\end{proof}

Lemma \ref{lemma.2} bounds the regret when the active server set is time-invariant. 
%Based on which we link Lemma \ref{lemma.2} to time-varying server sets and thus the nontrivial regret in \eqref{eq.sreg}.
Similar to \cite{kanade2009,li2018}, with the instantaneous regret of device $j$ defined as $V_t^j\big({\cal K}_t\big):= \sum_{k=1}^{K} {p}_t^{j}(k) r_t^j (k) - \Phi^{j*}({\cal K}_t^j)$,
%\begin{align}
%V_t^j\big({\cal K}_t\big):= \sum_{k=1}^{K} {p}_t^{jk} r_t^j (k) - \Phi^{j*}({\cal K}_t^j). 
%\end{align} 
the key step is to decouple the expected regret as
\begin{align}\label{eq.corollary1-1}
&\sum_{t=1}^{T}\mathbb{E}\Big[V_t^j\big({\cal K}_t\big)\Big] = \sum_{t=1}^{T}\sum_{\tilde{\cal K} \subseteq {\cal K}} \mathbb{P}({\cal K}_t^j = \tilde{\cal K}) \mathbb{E}\Big[ V_t^j\big({\cal K}_t\big) \big{|} {\cal K}_t^j = \tilde{\cal K}  \Big] \nonumber \\
& = \sum_{\tilde{\cal K} \subseteq {\cal K}} \mathbb{P}({\cal K}_t^j = \tilde{\cal K}) \sum_{t=1}^{T} \mathbb{E}\Big[ V_t^j\big(\tilde{\cal K}\big) \big{|} {\cal K}_t^j = \tilde{\cal K}  \Big] \nonumber \\
&\stackrel{(a)}{=} \sum_{\tilde{\cal K} \subseteq {\cal K}} \mathbb{P}({\cal K}_t^j = \tilde{\cal K}) \mathbb{E}\Bigg[\sum_{t=1}^{T} \bigg( \mu_t^j+ \frac{\eta_t^j}{2} \bigg) Q_t^j + \frac{\ln K}{\eta_{T+1}^j} \bigg{|} {\cal K}_t^j = \tilde{\cal K} \Bigg] \nonumber \\
&=\mathbb{E}\Bigg[\sum_{t=1}^{T} \bigg( \mu_t^j+ \frac{\eta_t^j}{2} \bigg) Q_t^j + \frac{\ln K}{\eta_{T+1}^j} \Bigg]
\end{align}
where in (a) we used the result of Lemma \ref{lemma.2}.

\subsection{Bound of $Q_t^j$ for stochastic jamming attacks}\label{appendix.sQ}
To evaluate the regret bound in Theorem \ref{theo.1}, we provide a bound on $Q_t^j$ in the following lemma.

\begin{lemma}\label{lemma.1}
	If $\mu_t^j \leq 1$ for every $t$, then $Q_t^j$ is bounded by
	\begin{align}
		\frac{1}{1+\mu_t^j}\leq Q_t^j \leq \alpha_t^j  + \sum_{k \in {\cal S}_t^j}  p_t^j(k) - \sum_{k \in {\cal S}_t^j}\frac{\mu_t^j p_t^j(k)}{2}
	\end{align}
	where $\alpha_t^j$ is the independence number of ${\cal G}_t^j$.
\end{lemma}

\begin{proof}
Recall the definition of $Q_t^j$ in \eqref{eq.def-q}.
%\begin{equation}
%	Q_t^j:= \sum_{k=1}^{K}\frac{p_t^j(k)}{\mu_t^j+\sum_{(m, k) \in {\cal G}_t^j} p_t^j(m)}.
%\end{equation}
If $k \in {\cal S}_t^j$, we have ${p_t^j(k)}/\big(\mu_t^j+\sum_{(m, k) \in {\cal G}_t^j} p_t^j(m) \big) = p_t^j(k)/(1+\mu_t^j)$, while for $k \notin {\cal S}_t^j$, we have  
${p_t^j(k)}/\big(\mu_t^j+\sum_{(m, k) \in {\cal G}_t^j} p_t^j(m) \big) = p_t^j(k)/(p_t^j(k)+\mu_t^j)$. 
%The following lemma renders useful upper and lower bounds on $Q_t^j$. 
First, we readily find that
\begin{equation}
\!Q_t^j\!=\!\sum_{k=1}^{K}\frac{p_t^j(k)}{\mu_t^j+\sum_{(m, k) \in {\cal G}_t^j} p_t^j(m)} \geq \sum_{k=1}^{K}\frac{p_t^j(k)}{\mu_t^j+1}\!=\!\frac{1}{1+\mu_t^j}.\!\!\!
\end{equation}

On the other hand, for $k \in {\cal S}_t^j$, it holds that ${p_t^j(k)}/\big(\mu_t^j+\sum_{(m, k) \in {\cal G}_t^j} p_t^j(m) \big) = p_t^j(k)/(1+\mu_t^j)$; while for $k \notin {\cal S}_t^j$, we have  
${p_t^j(k)}/\big(\mu_t^j+\sum_{(m, k) \in {\cal G}_t^j} p_t^j(m) \big) = p_t^j(k)/(p_t^j(k)+\mu_t^j)$. Hence,
\begin{align}
	Q_t^j  &= \sum_{k \in {\cal S}_t^j} \frac{p_t^j(k)}{1+\mu_t^j} + \sum_{k \notin {\cal S}_t^j} \frac{p_t^j(k)}{p_t^j(k)+\mu_t^j} \nonumber \\
	& \stackrel{(a)}{\leq} \sum_{k \in {\cal S}_t^j} p_t^j(k) \big(1- \frac{\mu_t^j}{2}\big) + \sum_{k \notin {\cal S}_t^j} \frac{p_t^j(k)}{p_t^j(k) + \mu_t^j} \nonumber \\
	&\leq \alpha_t^j  + \sum_{k \in {\cal S}_t^j}  p_t^j(k) - \sum_{k \in {\cal S}_t^j}\frac{\mu_t^j p_t^j(k)}{2}
\end{align}
where (a) uses the inequality $\frac{1}{1+x}\leq 1-\frac{x}{2}, \forall x \in [0,1]$.
%It should be pointed out as well that another bound that $Q_t^j \leq 2 \alpha_t^j \ln \big(1+\frac{\lceil K^2/\mu_t \rceil+ K}{\alpha_t}\big)$ is also derived in \cite{kocak2014}, but their bound is rather loose.
\end{proof}

Lemma \ref{lemma.1} can be adopted to bound the cooperation value $\lambda^j$. Since for the considered graph we have $\alpha_t^j = K - |{\cal S}_t^j|$, the upper bound of $Q_t^j$ in Lemma \ref{lemma.1} can be rewritten as $Q_t^j \leq \min \{ K, K - |{\cal S}_t^j|+1\}$. Then plugging it into the definition of $\lambda^j$, and using the fact that $\delta \leq K$, we have
\begin{align}
	\lambda^j \leq \sqrt{\frac{1}{T} + \frac{1}{KT}\sum_{t=1}^T \min \left\{K, K+1 - |{\cal S}_t^j| \right\}}.
\end{align}
Summing over $j$, yields the upper bound of $\lambda$.

\subsection{Proof of Corollary \ref{col.fixstepsize}}\label{appendix.sfdstep}
Without cooperation, it clearly holds that
\begin{equation}
	Q_t^j = \sum_{k=1}^K \frac{p_t^j(k)}{p_t^j(k)+ \mu_t^j} \leq  K.
\end{equation}
For $\eta_t^j =\sqrt{\frac{\ln K}{KT}}$ and $\mu_t^j = \frac{\eta_t^j}{2}$, \eqref{eq.theo1} becomes
\begin{align}
\mathbb{E} \big[\text{Reg}_T^j \big] \leq \eta_{T+1}^j KT + \frac{\ln K}{\eta_{T+1}^j} = 2\sqrt{TK\ln K}.
\end{align}

On the other hand, if $\eta_t^j =\sqrt{\frac{\ln K}{2Kt}}$ and $\mu_t^j = \frac{\eta_t^j}{2}, \forall t$, and the independence between $\eta_t^j$ and $Q_t^j$ implies
\begin{align}
\mathbb{E} \big[\text{Reg}_T^j \big] \leq K\sum_{t=1}^{T}\eta_t^j + \frac{\ln K}{\eta_{T+1}^j} \leq 2\sqrt{2TK\ln K},
\end{align}
where the inequality follows since $\sum_{t=1}^{T} 1/\sqrt{t} \leq 2\sqrt{T}$.

\vspace{-0.4cm}

\subsection{Proof of Corollary \ref{col.adastepsize}}\label{appendix.sadastep}

The proof builds on the following lemma.
\begin{lemma}\label{lemma.3}
	With $Q_1, Q_2, \cdots, Q_T$ and $K$ denoting positive real numbers, the following inequality holds 
	\begin{align}
	\sum_{t=1}^{T} \frac{Q_t}{2\sqrt{\delta + \sum_{\tau=1}^{t} Q_\tau} } 
	\leq \sqrt{\delta + \sum_{t=1}^{T} Q_t} - \sqrt{\delta}.
	\end{align}
\end{lemma}
\begin{proof}
	For $x\leq 1$, we have the inequality $\frac{x}{2} \leq 1 - \sqrt{1-x}$.
%	\begin{equation}
%	\frac{x}{2} \leq 1 - \sqrt{1-x}
%	\end{equation}
	Replacing $x$ with $Q_t / \big( \delta + \sum_{\tau=1}^{t} Q_\tau  \big)\leq 1$, we have 
	\begin{equation}
	\frac{Q_t}{2\big( \delta + \sum_{\tau=1}^{t} Q_\tau  \big)} \leq 1 - \sqrt{1-\frac{Q_t}{ \delta + \sum_{\tau=1}^{t} Q_\tau }}.
	\end{equation}
	Then multiplying both sides with $ \!\sqrt{\delta \!+\! \sum_{\tau=1}^{t} \!Q_\tau}$, we arrive at 
	\begin{equation}
	\frac{Q_t}{2\sqrt{ \delta\! +\! \sum_{\tau=1}^{t} Q_\tau}} \!\leq\! \sqrt{\delta\! +\! \sum_{\tau=1}^{t} Q_\tau} \!-\! \sqrt{{\delta \!+\! \sum_{\tau=1}^{t-1} Q_\tau}}.
	\end{equation}
	Taking summation over $T$, completes the proof.
\end{proof}

We are ready to prove the corollary. For a specific realization of $Q_t^j$, upon choosing $\eta_t^j = \sqrt{ (\ln K) / \big( K+ \sum_{\tau = 1}^{t-1}Q_\tau^j\big)}$, and $\mu_t^j =\eta_t^j/2$, we arrive at
\begin{align}\label{eq.kocak-fault}
	& ~~~~\sum_{t=1}^{T} \bigg( \mu_t^j+ \frac{\eta_t^j}{2} \bigg) Q_t^j = \sum_{t=1}^{T} \frac{Q_{t}^j\sqrt{ \ln K}}{ \sqrt{K + \sum_{\tau = 1}^{t-1}Q_\tau^j  }}  \nonumber \\
	& \leq  \sum_{t=1}^{T} \frac{Q_{t}^j \sqrt{\ln K}}{ \sqrt{K -Q_t^j + \sum_{\tau = 1}^{t}Q_\tau^j  }} \stackrel{(a)}{\leq} \sum_{t=1}^{T} \frac{Q_{t}^j \sqrt{\ln K}}{ \sqrt{\delta + \sum_{\tau = 1}^{t}Q_\tau^j  }} \nonumber \\
	& \stackrel{(b)}{\leq} \sqrt{ \bigg(\delta + \sum_{t=1}^{T} Q_t \bigg) \ln K} 
\end{align}
where (a) uses $\delta := \min_t \{ K- Q_t^j \}$ which is strictly greater than $0$ according to Lemma \ref{lemma.1}; and (b) follows from Lemma \ref{lemma.3}. Then, it is easy to see that
\begin{align}\label{eq.sreg_aux1}
	\sum_{t=1}^{T} \bigg( \mu_t^j+ \frac{\eta_t^j}{2} \bigg) Q_t^j + \frac{\ln K}{\eta_{T+1}^j} \leq 2\sqrt{ \bigg(\delta + \sum_{t=1}^{T} Q_t \bigg) \ln K}
\end{align}
Taking $\mathbb{E}$ w.r.t. the servers on/off probability on \eqref{eq.sreg_aux1}, we complete the proof. The bound \eqref{eq.sreg_aux1} can be approximated by $2\sqrt{\sum_{t=1}^{T} Q_t \ln K}$, since $\delta \ln K$ is not the dominant term.

\subsection{Proof of Lemma \ref{lemma.reformulation}}\label{appendix.sreform}
To show that $\big(\mathbf{q}_t^j \big)^\top \check{\mathbf{r}}_t^j = \big(\mathbf{p}_t^j\big)^\top \mathbf{r}_t^j$, it suffices to prove $ \big(\mathbf{q}_t^j\big)^\top \bm{\Gamma}\big({\cal K}_t^j\big) = \big(\mathbf{p}_t^j \big)^\top $. The key is the special structure of $\bm{\Gamma}\big({\cal K}_t^j\big) \in \{0,1\}^{\check{K}\times K}$. For $k \notin {\cal K}_t^j$, the $k$-th column is all $0$. And for $k \in {\cal K}_t^j$, the $k$-th column has $\check{K}/K$ entries equal to $1$ and other entries $0$. Besides, each row of $\bm{\Gamma}\big({\cal K}_t^j\big)$ has only one non-zero entry since each server list only has one output given ${\cal K}_t^j$. Without loss of generality, let rows $(m-1)\check{K}/K +1,\ldots \check{K}/K$ of $\bm{\Gamma}\big({\cal K}_t^j\big)$ be of the form $[0,\cdots,1,\cdots,0]$ with the $m$-th entry being $1$. Then $ \big(\mathbf{q}_t^j\big)^\top \bm{\Gamma}\big({\cal K}_t^j\big)\!=\!\big(\mathbf{p}_t^j \big)^{\!\top}$ becomes $\sum_{k= (m-1)\check{K}/K+1}^{m\check{K}/K} q_t^j(k) = p_t^j(m), \forall m \in {\cal K}$, which has at least one solution in $\Delta^{\check{K}}$.

\subsection{Bound of $Q_t^j$ for adversarial jamming attacks}\label{appendix.aQ}

Recall the definition of $Q_t^j$ for adversarial jammers in \eqref{eq.def-q2},
%\begin{equation}
%Q_t^j:= \sum_{k=1}^{\check{K}}\frac{q_t^j(k)}{\mu_t^j+\sum_{(m, k) \in {\cal G}_t^j} q_t^j(m)},
%\end{equation}
where ${\cal G}_t^j$ has $\check{K}$ nodes. Define node set ${\cal N}_1$ as the nodes satisfing $\Phi(\tilde{\cal K}_t^j) \in  {\cal S}_t^j$; and collect other nodes in set ${\cal N}_2$. Specifically, for $k \in {\cal N}_1$, we have ${q_t^j(k)}/\big(\mu_t^j+\sum_{(m, k) \in {\cal G}_t^j} q_t^j(m) \big) = q_t^j(k)/(1+\mu_t^j)$. Since the structure of ${\cal G}_t^j$ depends on the available set ${\cal K}_t^j$ and the side observation set ${\cal S}_t^j$, the next lemma presents ${\cal K}_t^j$ and ${\cal S}_t^j$ dependent bounds on $Q_t^j$.

\begin{lemma}\label{lemma.4}
	Quantity $Q_t^j$ satisfies the following
	\begin{align*}
		\frac{1}{1+\mu_t^j}\!\leq Q_t^j \leq \! |{\cal K}_t^j \cup {\cal S}_t^j\big|-\big|{\cal S}_t^j\big| + \mathds{1} \big({\cal S}_t^j \neq \emptyset \big).
	\end{align*}
\end{lemma}

\begin{proof}
Following steps similar to those in proving Lemma \ref{lemma.1}, it can be shown that $Q_t^j \geq \frac{1}{1+\mu_t^j}$. To derive the upper bound, consider first that there are no side observations, meaning ${\cal N}_1 = \emptyset$. Leveraging the definition of $Q_t^j$ and the symmetric structure of ${\cal G}_t^j$ in this case, it is easy to see that $Q_t^j \leq |{\cal K}_t^j|$.

Then consider ${\cal N}_1 \neq \emptyset$, which means that side observations are available. For the nodes in ${\cal N}_1$, we have
\begin{align}\label{eq.qpart1}
	\sum_{k \in {\cal N}_1 }\frac{q_t^j(k)}{\mu_t^j\!+\!\sum_{(m, k) \in {\cal G}_t^j} q_t^j(m)} =  \sum_{k \in {\cal N}_1 }\frac{q_t^j(k)}{\mu_t^j\!+\! 1} \stackrel{(a)}{\leq} 1,
\end{align}
where (a) uses the fact that $\sum_{k \in {\cal N}_1 } q_t^j(k) \leq 1$. Then leveraging the structure of ${\cal G}_t^j$, we have for the nodes of ${\cal N}_2$, 
\begin{align}\label{eq.qpart2}
&~~\sum_{k \in {\cal N}_2 }\frac{q_t^j(k)}{\mu_t^j\!+\!\sum_{(m, k) \in {\cal G}_t^j} q_t^j(m)} \leq  \sum_{{\cal I} } \sum_{k \in {\cal I}} \frac{ q_t^j(k)}{\sum_{(m, k) \in {\bar{\cal G}}_t^j} q_t^j(m)} \nonumber \\
&\leq \sum_{{\cal I} }\frac{ \sum_{k \in {\cal I}} q_t^j(k)}{\sum_{k \in {\cal I}} q_t^j(k)} \stackrel{(b)}{=} \big|{\cal K}_t^j \cup {\cal S}_t^j\big| - \big| {\cal S}_t^j\big|
\end{align}
where $\bar{\cal G}_t^j$ is the subgraph of ${\cal G}_t^j$ without unidirectional edges; ${\cal I}$ collects the cliques in $\bar{\cal G}_t^j$; and (b) uses $|{\cal I}| = \big|{\cal K}_t^j \cup {\cal S}_t^j\big| - \big| {\cal S}_t^j\big| $. Adding \eqref{eq.qpart1} and \eqref{eq.qpart2}, we obtain the upper bound on $Q_t^j$ for ${\cal N}_1 \neq \emptyset$. Writing the upper bound on $Q_t^j$ for both ${\cal N}_1 \neq \emptyset$ and ${\cal N}_1 = \emptyset$ compactly, completes the proof of the lemma.
\end{proof}

Simply plugging the results of Lemma \ref{lemma.4} into the definition of $\lambda^j$, we can upper bound $\lambda^j$ of SAVE-A as
\begin{align}
\!\!\lambda^j \!\leq\! \sqrt{\frac{1}{T} \!+\!\frac{1}{KT}\sum_{t=1}^T   \Big( \big|{\cal K}_t^j \cup {\cal S}_t^j\big|-\big|{\cal S}_t^j\big| + \mathds{1} \big({\cal S}_t^j \neq \emptyset \big)  \Big)}.
\end{align}
Summing over $j$, leads to the upperbound on $\lambda$.

%\bibliographystyle{IEEEtran}
%\bibliography{myabrv,datactr}

\end{document}